\documentclass[lettersize,journal]{IEEEtran}
\usepackage{amsmath,amsfonts,amssymb,amsthm,mathtools}
\usepackage{array}
\usepackage[caption=false,font=normalsize,labelfont=sf,textfont=sf]{subfig}
\usepackage{textcomp}
\usepackage{stfloats}
\usepackage{url}
\usepackage{verbatim}
\usepackage{graphicx}
\usepackage{cite}
\usepackage{booktabs}
\usepackage{multirow}
\usepackage{makecell}
\usepackage{siunitx}
\usepackage{float}
\usepackage{enumitem}
\usepackage{mathtools}
\usepackage{bm}
\usepackage{xcolor}

\newcommand{\Wass}{\mathcal{W}_1}
\newcommand{\Dt}{D_{\mathrm{task}}}

\theoremstyle{plain}
\newtheorem{theorem}{Theorem}[section]
\newtheorem{proposition}[theorem]{Proposition}
\newtheorem{corollary}[theorem]{Corollary}
\theoremstyle{remark}
\newtheorem{remark}[theorem]{Remark}
\newtheorem{lemma}[theorem]{Lemma}
\theoremstyle{definition}
\newtheorem{assumption}[theorem]{Assumption}

\begin{document}

\title{HIR-ALIGN: Enhancing Hyperspectral Image Restoration via Diffusion-Based Data Generation}

\author{Li Pang\textsuperscript{\dag}, Heng Zhao\textsuperscript{\dag}, Yijia Zhang, Deyu Meng, Xiangyong Cao\textsuperscript{*}
\thanks{Corresponding author: Xiangyong Cao (caoxiangyong@xjtu.edu.cn).}
}

\markboth{}%
{}

\IEEEpubid{}

\maketitle

\begin{abstract}
Hyperspectral image (HSI) restoration is crucial for reliable analysis, as real-world HSIs suffer from degradations like noise, blur, and resolution loss. However, existing models trained on source data often fail when deployed on target domains lacking clean references, a common occurrence in real-world scenarios. To address this issue, we present \textbf{HIR-ALIGN}, which enhances \underline{h}yperspectral \underline{i}mage \underline{r}estoration via \underline{a}ugmenting \underline{l}imited training \underline{i}mages by \underline{g}enerating synthetic data that closely matches the target distribution with \underline{n}o-extra clean target-domain HSI data, a plug-and-play target-adaptive augmentation framework. The framework consists of three stages: (i) \textbf{proxy generation}, where off-the-shelf restoration models are applied to degraded target observations to produce semantics-preserving proxy HSIs that approximate the target-domain clean images; (ii) \textbf{distribution-adaptive synthesis}, where a blur-robust unCLIP diffusion model generates target-aligned RGBs from proxy RGBs, with prompt conditioning and embedding-space noise initialization. Subsequently, a warp-based spectral transfer module synthesizes HSIs by aligning each generated RGB with the proxy RGB, estimating soft patch-wise transport weights, and applying the same weights and learnable local interpolation kernels to the proxy HSI; and (iii) \textbf{aligned supervised finetuning}, where restoration networks pretrained on the source distribution are finetuned using both the proxy HSIs and the synthesized target-aligned HSIs, and are then deployed on degraded target images. We further provide theoretical analysis showing that, under stated assumptions, the proposed augmentation-based finetuning can obtain a tighter target-domain restoration-risk upper bound
by jointly improving target-distribution coverage and controlling spectral bias. Extensive experiments on both simulated and real datasets across multiple restoration tasks, such as denoising and super-resolution, demonstrate that HIR-ALIGN achieves superior performance to proxy-only target-adaptation baselines and outperforms representative unsupervised methods in most cases.

\end{abstract}

\begin{IEEEkeywords}
hyperspectral image restoration, diffusion model, data augmentation, domain adaptation
\end{IEEEkeywords}

\section{Introduction}

Hyperspectral imaging records scene radiance over dozens of narrow spectral bands, providing substantially richer material and physical information than conventional RGB imaging. Owing to this dense spectral representation, hyperspectral images (HSIs) have been widely used in material identification \cite{li2012coupled,material1}, object detection \cite{imamoglu2018hyperspectral,object1,object2}, medical imaging \cite{yoon2022hyperspectral,medical1}, and remote sensing \cite{Remote1,Remote2}. In practical acquisition, however, HSIs are often corrupted by sensor noise, low spatial resolution, blur, missing bands, occlusions, and mixed degradations caused by hardware limitations or imaging environments. HSI restoration is therefore a prerequisite for reliable downstream analysis.

\begin{figure}[t]
\centering
\includegraphics[width=\linewidth]{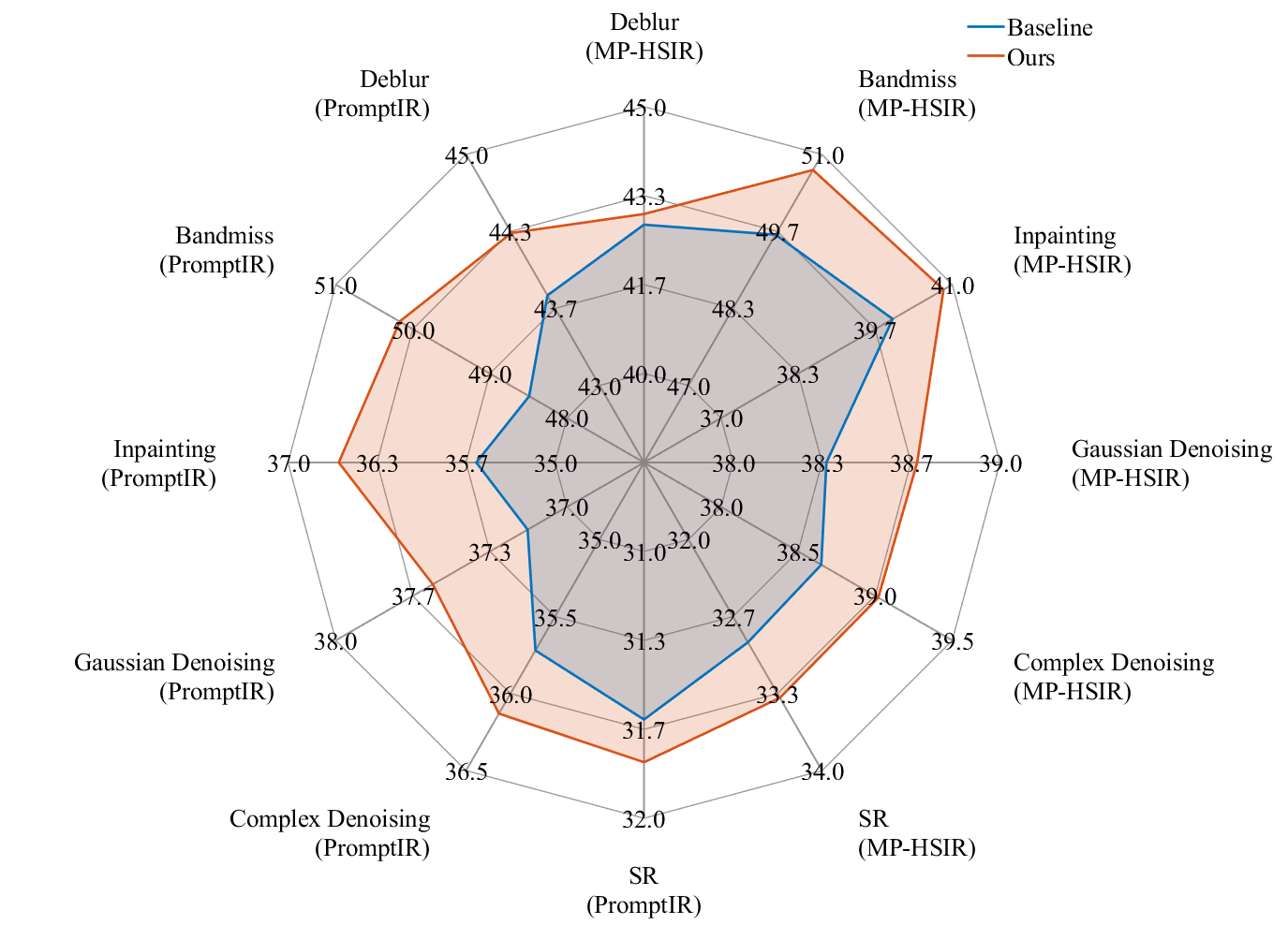}
\caption{PSNR comparison between proxy-only target-adaptive baselines and HIR-ALIGN-adapted models on CAVE \cite{CAVE}. All backbones are pretrained on ICVL \cite{ICVL}; the baseline is further finetuned using proxy data only, whereas HIR-ALIGN is finetuned using both proxy data and synthesized target-aligned data. X (Y) denotes the performance of model Y on task X.}
\label{fig:radar}
\end{figure}

The performance of an image restoration system is mainly determined by three coupled factors: the restoration model, the degradation model, and the training data distribution. The restoration model determines the representation capacity and inference mechanism of the restorer. The degradation model describes how clean HSIs are corrupted during acquisition or during synthetic training-data construction, and thus determines whether the learned inverse mapping is consistent with the target restoration task. The training data distribution determines whether the learned spatial--spectral priors match the scenes to be restored. While substantial progress has been made in architectures (e.g., traditional low-rank, sparsity, and tensor priors \cite{lowrank,sparsity,chen2022hyperspectral,chen2023hyperspectral}, as well as deep CNN, Transformer, Mamba, and diffusion networks \cite{Efficiency,pang2022trq3dnet,pan2023progressive,SERT,SSUMamba,MP-HSIR}) and realistic degradation modeling (e.g., non-i.i.d.\ mixture-of-Gaussians, band-dependent, asymmetric, and mixed noises \cite{chen2018noniid,yue2018complexnoise,ma2021robustnoise,xu2022asymmetricnoise,jiang2018adaptive,zhuang2023fasthymix}), these advancements do not resolve the data-distribution mismatch between training (source) and deployment (target) domains. Shifts in scene content, illumination, or spectral statistics render source-learned spatial--spectral priors suboptimal. This distribution-shift problem is particularly severe for HSI restoration. In many real scenarios, one can easily collect degraded target observations from the deployed sensor, but obtaining the corresponding clean target HSIs is expensive, time-consuming, or physically impossible. Therefore, the central question considered in this paper is: \textbf{\emph{how can we transform a small set of degraded target observations into effective target-aligned supervision, such that a source-trained restorer can achieve better restoration on the observed target data and generalize more reliably to unknown samples from the same target domain, without using additional clean target-domain data?}}

Although generative diffusion models excel at data augmentation, their success has been extensively demonstrated in natural-image synthesis. Representative models, including DDPM \cite{DDPM}, DDIM \cite{DDIM}, score-based diffusion models \cite{DM-Langevin,scores}, ADM \cite{classifier-DM}, GLIDE \cite{GLIDE}, latent diffusion models \cite{latent_diffusion}, unCLIP \cite{unCLIP}, Imagen \cite{imagen}, and DiT \cite{DiT}, have achieved high-fidelity unconditional, class-conditional, text-conditioned, and image-variation generation on RGB natural images. Nevertheless, these models mainly learn RGB appearance priors from large-scale natural-image or text--image data and do not directly provide spectrally faithful HSI supervision. Meanwhile, existing remote-sensing and HSI generators (e.g., CRS-Diff \cite{tang2024crs}, AeroGen \cite{tang2024aerogen}, TerraGen \cite{TerraGen}, HSIGene \cite{pang2024hsigene}) cannot directly solve this target-adaptive problem due to two key limitations:
\begin{itemize}
    \item \textbf{Domain and Modality Specificity:} Natural-image diffusion models mainly operate in the RGB image space and lack explicit spectral-generation constraints, while remote-sensing and HSI generators are constrained to specific domains, layouts, or training corpora. As a result, they struggle to synthesize arbitrary target distributions, such as close-range natural scenes (e.g., CAVE and KAIST), while preserving reliable hyperspectral signatures.
    \item \textbf{Reliance on High-Quality Data:} They typically require clean images, paired supervision, high-quality conditions, or large-scale curated training data. Directly applying these generators to degraded target observations would merely reproduce artifacts (e.g., noise, blur, missing bands) rather than providing clean, target-aligned hyperspectral supervision.
\end{itemize}

To address these issues, we propose \textbf{HIR-ALIGN}, a target-adaptive augmentation framework for HSI restoration under distribution shift. The framework contains three stages: \textbf{proxy generation, semantic-guided distribution-adaptive synthesis, and aligned supervised finetuning}. In the first stage, a source-pretrained restorer is applied to degraded target observations to obtain proxy HSIs. Although these proxies are not perfect ground truth, they remove a large portion of the target degradation and preserve the dominant scene structures and spectral characteristics of the target domain. In the second stage, a blur-robust conditional diffusion model modified from unCLIP \cite{unCLIP} is used to generate target-aligned RGB images from the proxy RGBs. Compared with the original unCLIP pipeline, we introduce a degradation-robust CLIP model, prompt conditioning and embedding-space noise initialization to improve the robustness of RGB synthesis when the conditioning proxies still contain residual restoration errors. To convert the generated RGBs into HSIs, we design a warp-based spectral transfer module. Specifically, we first estimate aggregation weights and interpolation kernels by matching each generated RGB to the proxy RGB through patch descriptors, candidate retrieval, soft pixel aggregation, and learnable local interpolation. These estimated parameters are then applied to the proxy HSI. Consequently, the synthesized HSI inherits authentic spectra from the proxy HSI while adopting the spatial layout and appearance distribution of the generated RGB. This constrained spectral inheritance increases the diversity of target-domain supervision without freely hallucinating spectra, thereby reducing the risk of spectral drift. In the third stage, both proxy samples and synthesized target-aligned samples are used to finetune existing restoration backbones, enabling the adapted models to improve restoration on the observed target images and generalize more reliably to unseen samples from the same target domain. 
This data-construction strategy improves performance because it simultaneously reduces overfitting to a small observed target set, increases target coverage through semantic-guided views, and suppresses spectral drift
through proxy-anchored spectral inheritance. 

The main contributions of this paper are summarized as follows.
\begin{itemize}[leftmargin=*,itemsep=1.5pt,topsep=2pt]
    \item We propose HIR-ALIGN, a plug-and-play data-distribution alignment framework that improves existing supervised HSI restorers through proxy generation, distribution-adaptive synthesis, and aligned supervised finetuning. The framework emphasizes data distribution as a key factor of restoration performance, complementary to model design and degradation simulation.
    \item We design a semantic-guided RGB synthesis module based on improved unCLIP for target-domain augmentation. Specifically, we enhance the original unCLIP framework by introducing a degradation-robust CLIP, prompt conditioning and embedding-space noise initialization, which improve the robustness of conditional diffusion to residual degradations and structural inaccuracies in proxy RGBs, and enable the generation of RGB samples that better match the target-domain appearance distribution.
    \item We propose a warp-based hyperspectral generation module to convert the synthesized RGB images into target-aligned HSIs. Instead of relying on an unconstrained RGB-to-HSI reconstruction network, the proposed module transfers spectra from proxy HSIs to generated RGB layouts through patch matching, sparse candidate aggregation, and shared local interpolation, thereby avoiding spectral hallucination and providing more stable supervision.
    \item We provide a theoretical analysis to justify our data augmentation framework. By formalizing the hyperspectral image synthesis process as a sparse stochastic transport, we derive conditional target-risk bounds. The analysis explains how finetuning on a mixture of proxy and generated samples can expand target-distribution coverage while keeping proxy and warp errors controlled, and it identifies sufficient conditions under which the resulting upper bound is tighter than that of source-domain training.
    \item Extensive experiments on both simulated and real datasets across multiple restoration tasks, such as denoising and super-resolution, demonstrate that existing HSI restoration methods enhanced by HIR-ALIGN achieve better performance than  proxy-only target-adaptation baselines and outperform representative unsupervised methods and generative augmentation baselines in most cases.
\end{itemize}

\section{Related Works}
\subsection{Hyperspectral Image Restoration}

Hyperspectral image restoration aims to recover clean, high-quality data from degraded observations and is typically studied under unsupervised and supervised paradigms. Classical unsupervised methods mainly rely on manually designed priors. Representative examples include LRTV \cite{LRTV}, which combines low-rank decomposition with total variation, LRTFL0 \cite{LRTFL0}, which improves edge preservation through $L_0$ gradient regularization, Enhanced-3DTV \cite{Enhanced-3DTV}, which alleviates the sparsity limitation of traditional 3DTV by subspace learning, and RCTV \cite{RCTV}, which regularizes a representative coefficient matrix for improved efficiency. These methods are attractive because they do not require paired training data, yet their hand-crafted priors are often insufficient to capture the rich and highly nonstationary spatial--spectral structures encountered in real target domains. Consequently, their restoration performance typically falls behind that of supervised methods in in-domain scenarios. Furthermore, these optimization-based approaches generally involve high computational time complexity during the solving process and require laborious hyperparameter tuning.

Supervised restoration methods learn powerful spatial--spectral priors directly from data. For denoising, GRNNet \cite{GRNNet} models both local and global information, TRQ3DNet \cite{pang2022trq3dnet} bridges CNNs and Transformers for improved feature preservation, SERT \cite{SERT} enhances spatial--spectral interactions with a tailored Transformer architecture, and SSUMamba \cite{SSUMamba} adopts the Mamba architecture for efficient long-range modeling. For super-resolution, GDRRN \cite{GDRRN}, SFCSR \cite{SFCSR}, and ESSAformer \cite{ESSAformer} are strong supervised baselines, while diffusion-based variants such as HSR-Diff \cite{HSI_Diff} and ISPDiff \cite{ISPDiff} further explore stochastic generation. More general-purpose restoration backbones such as PromptIR \cite{promptir} and MP-HSIR \cite{MP-HSIR} aim to handle multiple degradations with a unified architecture.

Despite these advances, most supervised methods are still trained under a closed-world assumption: the training and testing domains are expected to share similar degradation statistics, scene composition, and spectral distribution. This assumption is often violated in practical HSI acquisition. Once the target domain shifts away from the source domain, source-pretrained restorers may exhibit large performance drops, even when their in-domain benchmark results are strong. Our work is complementary to the development of restoration backbones: instead of proposing yet another architecture for a fixed task, HIR-ALIGN provides a front-end target adaptation strategy that can be plugged into existing restorers without modifying their inference networks.

\subsection{Diffusion Models}
Diffusion models generate data by reversing a forward perturbation process that gradually maps a clean sample $x_0$ to Gaussian noise $x_t$:
\begin{equation}
q(x_t \mid x_{t-1}) = \mathcal{N}\!\left(x_t; \sqrt{1-\beta_t}\,x_{t-1},\, \beta_t I\right),
\end{equation}
and learn the reverse process
\begin{equation}
\label{eq:reverse}
p_\theta(x_{t-1} \mid x_t) = \mathcal{N}\!\left(x_{t-1} \mid \mu_\theta(x_t,t),\,\Sigma_\theta(x_t,t)\right).
\end{equation}
DDPM \cite{DDPM} and its extensions \cite{DDIM,classifier-DM,scores,DM-Langevin,latent_diffusion} have made diffusion models a dominant family of generative models. In natural-image synthesis, early score-based and denoising diffusion models \cite{DM-Langevin,DDPM,scores} established high-quality generation through iterative denoising or reverse-time stochastic dynamics, while DDIM \cite{DDIM} improved sampling efficiency by introducing a non-Markovian sampling process. ADM \cite{classifier-DM} further improved class-conditional ImageNet synthesis and demonstrated that diffusion models can surpass GAN-based methods in image fidelity. With classifier-free guidance \cite{ho2022classifier} and large-scale text--image pretraining, GLIDE \cite{GLIDE}, latent diffusion models \cite{latent_diffusion}, unCLIP \cite{unCLIP}, and Imagen \cite{imagen} enabled photorealistic text-conditioned or image-conditioned generation. More recently, DiT \cite{DiT} showed that transformer backbones can scale diffusion models effectively for high-quality image synthesis. These natural-image diffusion models provide powerful RGB appearance priors and motivate our use of unCLIP-style image-variation synthesis. The unCLIP framework \cite{unCLIP} is particularly attractive for our problem because it conditions image synthesis on CLIP embeddings and can therefore generate semantically related images from visual prompts rather than only from pure text. However, existing natural-image diffusion models, including unCLIP, are designed primarily for RGB image synthesis and do not explicitly model hyperspectral spectral correlations, making their outputs insufficient as spectrally reliable supervision for HSI synthesis.

Recent works have also explored diffusion-based generation or augmentation in remote sensing and hyperspectral imaging. For example, CRS-Diff \cite{tang2024crs} studies controllable remote-sensing image generation, AeroGen \cite{tang2024aerogen} uses diffusion-driven generation to enhance remote-sensing object detection, and TerraGen \cite{TerraGen} focuses on layout-based remote-sensing data augmentation. In the hyperspectral field, HSIGene \cite{pang2024hsigene} develops a foundation model for HSI generation, and HIR-Diff \cite{pang2024hir} explores diffusion priors for unsupervised HSI restoration. These methods demonstrate the potential of generative models, but they do not directly solve the target-adaptive restoration problem studied here: target adaptation from degraded observations of an arbitrary natural-scene HSI domain. In this setting, simply invoking either an RGB natural-image generator or a domain-specific HSI generator can easily create samples that are visually plausible yet poorly aligned with the target scene distribution or with the degradation-specific semantics contained in the observed target data.

\begin{figure*}[t]
\centering
\includegraphics[width=0.93\linewidth]{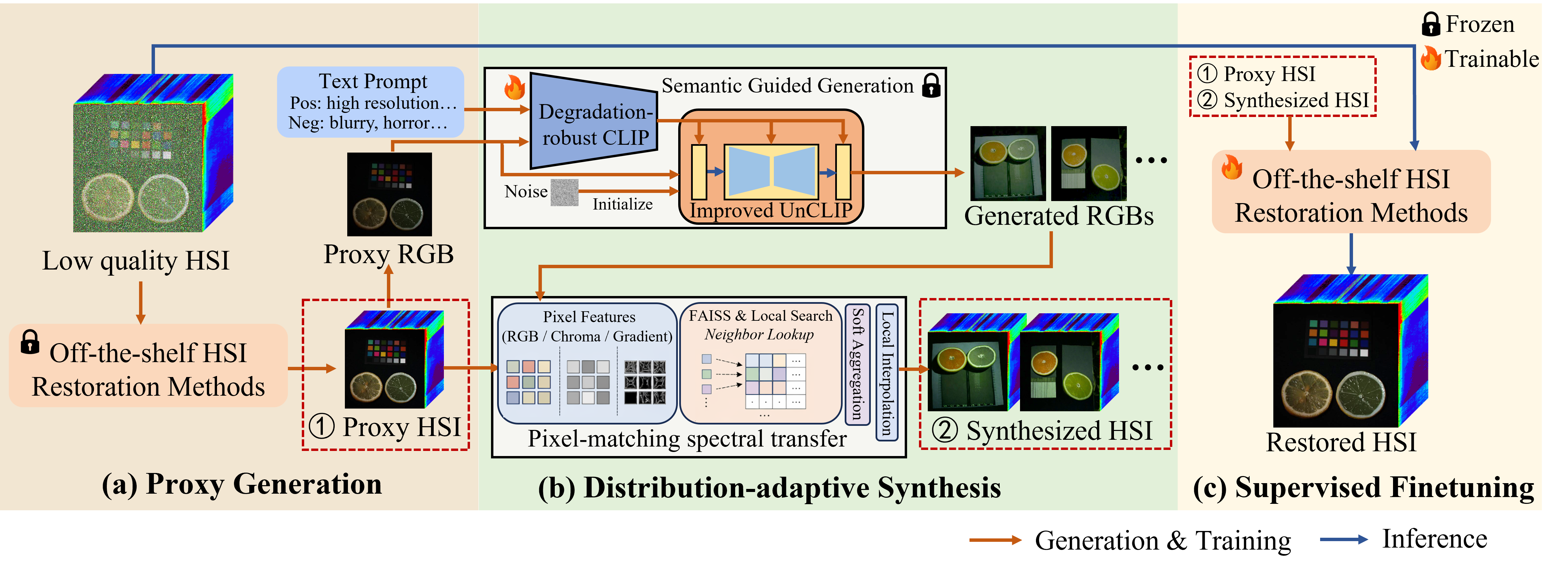}
\caption{Overview of HIR-ALIGN. \textbf{(a) Proxy generation:} source-pretrained HSI restorers produce proxy HSIs and proxy RGBs from degraded target observations. \textbf{(b) Distribution-adaptive synthesis:} improved unCLIP generates target-aligned RGBs, and the proposed pixel-matching spectral transfer module maps proxy spectra to the generated RGB layout by extracting pixel features from patch descriptors, retrieving candidates via neighborhood lookup, and applying soft aggregation with shared local interpolation. \textbf{(c) Aligned supervised finetuning:} re-degraded proxy pairs and synthesized HSI pairs are jointly used to finetune off-the-shelf restoration networks.}
\label{fig:framework}
\end{figure*}
\section{Method}

HIR-ALIGN contains three stages: proxy generation, distribution-adaptive synthesis, and aligned supervised finetuning. The first stage extracts a stable target-domain anchor from each degraded observation. The second stage enlarges the target support by generating target-aligned RGB views and converting them into HSIs through spectral warping. The third stage transforms these data into supervised training pairs and uses them to finetune an arbitrary restoration backbone. The full workflow is summarized in Fig.~\ref{fig:framework}.

\subsection{Proxy Generation}
For each degraded target observation $x_i$, we first obtain a proxy HSI $p_i$ by applying a source-pretrained restoration model $\Phi_{\mathrm{src}}$:
\begin{equation}
p_i=\Phi_{\mathrm{src}}(x_i).
\label{eq:proxy_def}
\end{equation}
The proxy is not expected to be a fully clean reconstruction. Instead, its role is to provide an initial target-domain estimate that preserves the scene layout, object boundaries, and dominant spatial--spectral structures of the latent target image. Therefore, $p_i$ serves as a pseudo-clean anchor for constructing target-adaptive training data. Using proxies rather than raw degraded data as synthesis anchors prevents the entanglement of semantics with artifacts by providing a clean data manifold, enabling effective supervised adaptation.

To couple the proxy HSI with the RGB diffusion model, we convert $p_i$ into a proxy RGB image $r_i=\Gamma(p_i)$. This is performed by selecting three visible-wavelength bands corresponding to red, green, and blue:
\begin{equation}
r_i=\Gamma(p_i)
=
\big[
p_i^{(b_R)},
p_i^{(b_G)},
p_i^{(b_B)}
\big],
\label{eq:proxy_rgb}
\end{equation}
where $b_R$, $b_G$, and $b_B$ denote the bands centered at, or closest to, 660nm, 550nm, and 470nm, respectively. The resulting pair $(p_i,r_i)$ is then used as the target-domain anchor in the subsequent synthesis stage.

Although the proxy might be imperfect, the proposed framework imposes less stringent requirements on them. The proxy mainly needs to preserve sufficient scene semantics and spatial structure for the RGB generator and the subsequent target-adaptive data construction. Moderate restoration errors in $p_i$ do not necessarily destroy this information, because the generation stage is relatively robust to local artifacts and low-level noise. Thus, even a relatively weak proxy can still provide useful target-domain structure for adaptation. 

\subsection{Distribution-Adaptive Synthesis}

This stage contains two coupled components: \textbf{target-aligned RGB generation} and \textbf{warp-based hyperspectral generation}. The first component introduces appearance diversity aligned with the target domain. The second component converts each generated RGB into an HSI while preserving spectral consistency with the proxy HSIs, thereby producing more stable and target-aligned supervision for subsequent finetuning. The details are as follows:



\textbf{1) Target-Aligned RGB Generation with Improved unCLIP.} Given a proxy RGB image $r_i\in[0,1]^{H\times W\times 3}$ for the $i$-th target sample, where $H$ and $W$ denote the spatial height and width, we adopt unCLIP \cite{unCLIP} to generate diverse RGB samples whose spatial appearance matches the target domain. Let $\mathbf{z}_I=E_{\mathrm{clip}}(r_i)$ denote the CLIP image embedding extracted from $r_i$, where $E_{\mathrm{clip}}$ is the CLIP vision encoder. Let $\mathbf{z}_T$ denote the text-conditioning embedding encoded from the fixed prompt guidance. They are fused into a conditioning code
\begin{equation}
\mathbf{z}_C = f_{\mathrm{cond}}(\mathbf{z}_I,\mathbf{z}_T),
\label{eq:cond-fusion}
\end{equation}
where $f_{\mathrm{cond}}(\cdot)$ is the conditioning fusion function and $\mathbf{z}_C$ is the final condition used by the reverse diffusion sampler. Following the DDPM formulation \cite{DDPM}, the mean of the Gaussian reverse transition is
\begin{equation}
\boldsymbol{\mu}_\theta\!\left(\mathbf{x}_t,t\mid\mathbf{z}_C\right)
=
\frac{1}{\sqrt{\alpha_t}}
\left(
\mathbf{x}_t-
\frac{1-\alpha_t}{\sqrt{1-\bar{\alpha}_t}}
\hat{\epsilon}_\theta\!\left(\mathbf{x}_t,t\mid\mathbf{z}_C\right)
\right).
\label{eq:ddpm-mean}
\end{equation}
Here, $\mathbf{x}_t$ is the noisy image latent at diffusion timestep $t$, $\boldsymbol{\mu}_\theta(\cdot)$ is the predicted mean for sampling $\mathbf{x}_{t-1}$, $\theta$ denotes the parameters of the denoising network, $\hat{\epsilon}_\theta(\mathbf{x}_t,t\mid\mathbf{z}_C)$ is the predicted noise conditioned on $\mathbf{z}_C$, $\alpha_t=1-\beta_t$ is determined by the noise variance schedule $\beta_t$, and $\bar{\alpha}_t=\prod_{s=1}^{t}\alpha_s$ is the cumulative noise-retention coefficient.

Directly using the original unCLIP conditioning is not sufficient because the proxy images may still contain blur, residual noise, or missing details. We therefore introduce three modifications:
\begin{itemize}
    \item \textbf{Vision Encoder Fine-tuning:} The CLIP vision encoder is fine-tuned to align embeddings of blurred images with those of their clean counterparts. Denoting the encoder by $E_{\mathrm{clip}}$, we minimize:
    \begin{equation}
    \mathcal{L}_{\mathrm{clip}} = \left\| E_{\mathrm{clip}}(I_{\mathrm{blur}}) - E_{\mathrm{clip}}(I_{\mathrm{clean}}) \right\|_2^2,
    \label{eq:clip_align}
    \end{equation}
    where $I_{\mathrm{blur}}$ and $I_{\mathrm{clean}}$ are a blurred image and its clean counterpart, respectively, $\mathcal{L}_{\mathrm{clip}}$ is the embedding-alignment loss, and $\|\cdot\|_2$ denotes the Euclidean norm. This objective encourages degraded conditions to remain semantically informative in the CLIP embedding space.
    \item \textbf{Prompt Guidance:} Fixed positive and negative prompts are injected as additional guidance to encourage high-quality details and suppress obvious artifacts. 
    \item \textbf{SDEdit-style Initialization:} Instead of starting the reverse diffusion sampling process from pure Gaussian noise, we follow an SDEdit-style initialization \cite{sdedit} by using a noised version of the proxy image latent as the sampling starting point. Let $\mathbf{x}_0^{\mathrm{init}}$ denote the initial clean latent obtained from the proxy RGB, and let $t_0$ be the selected starting timestep of the reverse sampler. We initialize the reverse diffusion trajectory at timestep $t_0$ by
    \begin{equation}
    \mathbf{x}_{t_0}^{\mathrm{init}}=
    \sqrt{\bar{\alpha}_{t_0}}\,\mathbf{x}_0^{\mathrm{init}}+
    \sqrt{1-\bar{\alpha}_{t_0}}\,\boldsymbol{\epsilon},
    \label{eq:sdedit_init}
    \end{equation}
    where $\mathbf{x}_{t_0}^{\mathrm{init}}$ is the initialized sampling state, $\boldsymbol{\epsilon}\sim\mathcal{N}(\mathbf{0},\mathbf{I})$ is standard Gaussian noise, and $\mathbf{I}$ is the identity covariance matrix. The sampler then denoises from $\mathbf{x}_{t_0}^{\mathrm{init}}$ to $\mathbf{x}_0$, so the generated image preserves the semantic structure of the proxy while still allowing diverse target-aligned appearance variations.

\end{itemize}

The output of this stage is a set of generated guide RGBs $\{g_{i,j}\}_{j=1}^{M}$ for each proxy sample, where $M$ is the number of generated RGB guides per proxy and $j$ indexes the $j$-th generated guide. 
By generating multiple RGB guides from each proxy, the synthesized training set can cover a broader portion of the target-domain appearance distribution, which is beneficial for improving adaptation and generalization by enriching local target-domain variability in texture, illumination, and fine geometry.

\begin{figure*}[t]
\centering
\includegraphics[width=\textwidth]{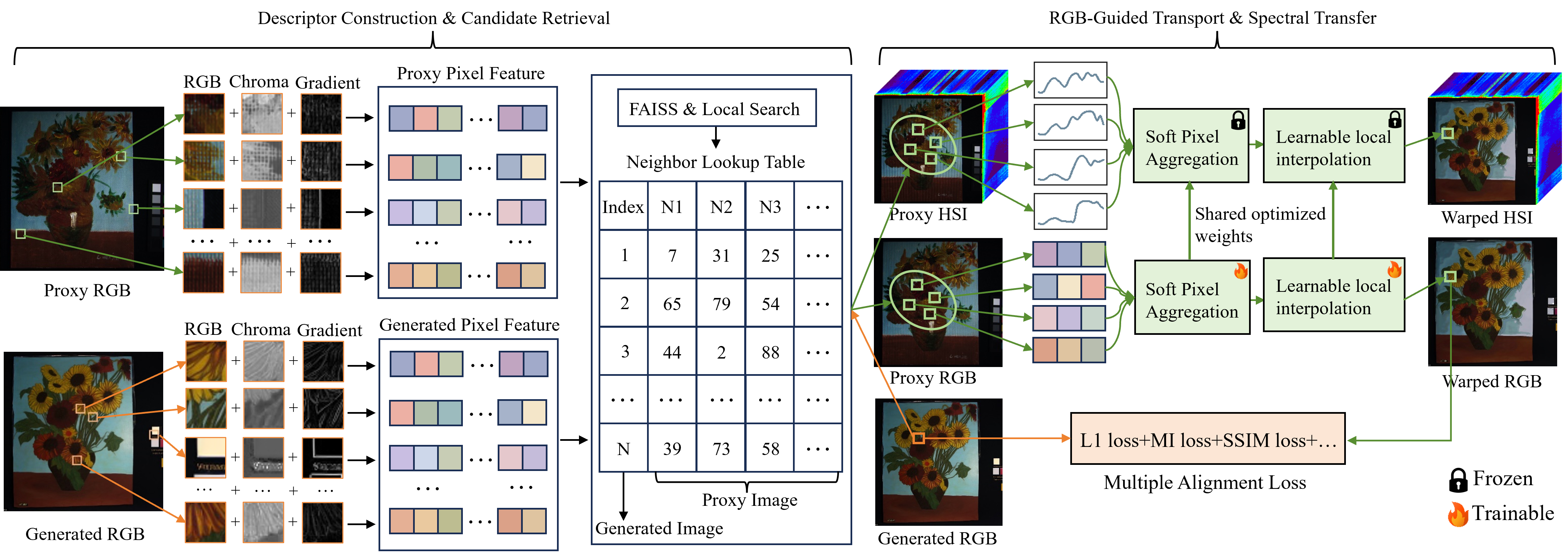}
\caption{Details of the warp based spectral transfer process. We first match the generated and proxy RGB images using multi feature patch descriptors and neighbor retrieval. We then optimize soft aggregation weights and local interpolation kernels guided by RGB alignment losses. Finally we apply these shared learned parameters directly to the proxy hyperspectral data to synthesize the target aligned output.}
\label{fig:warp}
\end{figure*}

\textbf{2) Warp-Based Hyperspectral Generation.} The synthesized RGB guide provides target-aligned appearance, while the proxy HSI provides the spectral anchor. We therefore synthesize a target-aligned HSI by estimating a sparse RGB-conditioned transport operator from the proxy RGB to the guide RGB and then applying the same operator to the proxy spectra. The process is shown in Figure \ref{fig:warp}. Let $i$ index a proxy sample and $j$ index one generated RGB guide associated with this proxy. We denote the proxy HSI by $p_i\in\mathbb{R}^{H\times W\times B}$, where $H$ and $W$ are the spatial height and width and $B$ is the number of spectral bands. The pseudo-RGB image derived from $p_i$ is denoted by $r_i\in[0,1]^{H\times W\times 3}$, and the generated guide RGB is denoted by $g_{i,j}\in[0,1]^{H\times W\times 3}$. The pixel lattice is $\Omega=\{1,\ldots,H\}\times\{1,\ldots,W\}$, and pixel coordinates are written as $u=(u_y,u_x)\in\Omega$ and $v=(v_y,v_x)\in\Omega$.

\textbf{Step1: Descriptor Construction and Candidate Retrieval.}
For both the proxy RGB and the guide RGB, we construct local patch descriptors that encode color, chromaticity, and structural information. Given an RGB image $I\in[0,1]^{H\times W\times 3}$, the descriptor at pixel $u$ is defined as
\begin{equation}
\begin{aligned}
\psi(I,u)
&=
\mathrm{Unfold}_{s_m}\!\Bigl(
[I,\lambda_c\chi(I),\lambda_g\nabla I]
\Bigr)_u,\\
\chi(I)(u)
&=
\frac{I(u)}
{\sum_{c=1}^{3}I_c(u)+\epsilon}.
\end{aligned}
\label{eq:descriptor}
\end{equation}
Here, $\psi(I,u)\in\mathbb{R}^{D_\psi}$ is the descriptor vector, $D_\psi$ is the descriptor dimension, $\mathrm{Unfold}_{s_m}(\cdot)_u$ extracts and vectorizes an $s_m\times s_m$ patch centered at $u$, $I_c(u)$ is the value of the $c$-th RGB channel at $u$, $\epsilon$ is a small positive constant for numerical stability, $\chi(I)$ is the chroma-normalized RGB image, $\nabla I$ denotes image-gradient features, and $\lambda_c$ and $\lambda_g$ are balancing coefficients. The chroma term reduces sensitivity to illumination variation, while the gradient term emphasizes local structural consistency.

For each guide pixel $u$, we retrieve a small set of appearance-consistent seed locations from the proxy image by nearest-neighbor search in the descriptor space. This retrieval is performed with FAISS~\cite{johnson2019billion}, which enables efficient similarity search over dense patch descriptors. Formally, the seed locations are defined as
\begin{equation}
\{q_{u,m}\}_{m=1}^{K_s}
=
\mathrm{NN}_{K_s}
\left(
\psi(g_{i,j},u),
\{\psi(r_i,v)\mid v\in\Omega\}
\right),
\label{eq:seed-retrieval}
\end{equation}
where $q_{u,m}\in\Omega$ is the $m$-th retrieved proxy seed for guide location $u$, and $\mathrm{NN}_{K_s}$ returns the $K_s$ proxy locations whose descriptors are closest to $\psi(g_{i,j},u)$ under the descriptor distance. Each seed is further expanded within a local spatial neighborhood,
\begin{equation}
\mathcal{N}_{\rho}(q)
=
\{v\in\Omega\mid \|v-q\|_{\infty}\le \rho\},
\quad
\mathcal{S}_{i,j}(u)
=
\bigcup_{m=1}^{K_s}\mathcal{N}_{\rho}(q_{u,m}),
\label{eq:local-candidate-pool}
\end{equation}
where $\rho$ is the refinement radius, $\mathcal{N}_{\rho}(q)$ is the neighborhood centered at seed $q$, and $\mathcal{S}_{i,j}(u)$ is the resulting candidate pool for guide pixel $u$. The final candidates are selected from the expanded candidate pool
$S_{i,j}(u)$ according to the mean absolute descriptor discrepancy.
For each guide pixel $u$ and each proxy candidate $v \in S_{i,j}(u)$,
we compute
\begin{equation}
\delta_{\psi}(u,v)=
\frac{1}{D_{\psi}}
\left\|
\psi(g_{i,j},u)-\psi(r_i,v)
\right\|_1,
\quad
v\in \mathcal{S}_{i,j}(u).
\label{eq:descriptor_distance}
\end{equation}
Therefore, $\delta_{\psi}(u,v)$ measures the average $\ell_1$ discrepancy between the guide descriptor at location $u$ and a candidate proxy descriptor at location $v$ drawn from the pool $S_{i,j}(u)$ in Eq.~(11).
We retain the $K$ candidates with the smallest descriptor distances:
\begin{equation}
C_{i,j}(u)=\{v_{u,1},\ldots,v_{u,K}\}\subseteq S_{i,j}(u),
\quad
d_{u,k}=\delta_{\psi}(u,v_{u,k}).
\end{equation} 
Here, $v_{u,k}$ is the $k$-th retained proxy candidate for guide pixel $u$, and $d_{u,k}$ is its corresponding descriptor distance. For notational simplicity, $\mathcal{C}_{i,j}(u)$ is written as a set, although repeated candidate locations may occur. Such repetitions only reduce the effective support of the sparse transport operator and do not change the formulation.

\textbf{Step2: RGB-Guided Estimation of the Sparse Transport Operator.}
The candidate set $\mathcal{C}_{i,j}(u)$ provides a sparse bank of proxy pixels that are visually compatible with guide location $u$. Rather than selecting a single hard correspondence from $\mathcal{C}_{i,j}(u)$, we estimate soft aggregation weights over all retained candidates. Let $\ell_{u,v}$ be the learnable aggregation logit associated with candidate proxy location $v\in\mathcal{C}_{i,j}(u)$, and let $\tau>0$ be the softmax temperature. The aggregation weight assigned to candidate $v$ is defined as
\begin{equation}
a_{u,v}(\ell)
=
\frac{\exp(\ell_{u,v}/\tau)}
{\sum_{v'\in \mathcal{C}_{i,j}(u)}
\exp(\ell_{u,v'}/\tau)},
\quad
v\in \mathcal{C}_{i,j}(u).
\label{eq:soft_aggregation}
\end{equation}
Here, $\ell=\{\ell_{u,v}\mid u\in\Omega,\; v\in\mathcal{C}_{i,j}(u)\}$ denotes all aggregation logits, and $a_{u,v}(\ell)$ is the normalized contribution of proxy pixel $v$ to guide pixel $u$. By construction, the weights over the candidate set satisfy $\sum_{v\in\mathcal{C}_{i,j}(u)}a_{u,v}(\ell)=1$. When the retained candidates are indexed as $v_{u,k}$, we use the shorthand $a_{u,k}(\ell):=a_{u,v_{u,k}}(\ell)$. The pre-interpolation RGB induced by these weights is
\begin{equation}
\bar{r}_{i,j}(u;\ell)
=
\sum_{v\in \mathcal{C}_{i,j}(u)}
a_{u,v}(\ell)\, r_i(v),
\quad
\sum_{v\in \mathcal{C}_{i,j}(u)} a_{u,v}(\ell)=1 .
\label{eq:pre_interp_rgb}
\end{equation}
where $r_i(v)\in\mathbb{R}^{3}$ is the RGB vector of the proxy image $r_i$ at candidate location $v$. 

To reduce local artifacts introduced by discrete candidate aggregation, we further estimate a local interpolation kernel. Let $\Delta_s=\{\eta=(\eta_y,\eta_x)\mid \eta_y,\eta_x\in[-\lfloor s/2\rfloor,\lfloor s/2\rfloor]\}$ denote an $s\times s$ spatial stencil, where $\eta$ is a local offset. Let $m_{u,\eta}$ be the interpolation logit for offset $\eta$ at pixel $u$. The interpolation weight is
\begin{equation}
b_{u,\eta}(m)
=
\frac{\exp(m_{u,\eta})}
{\sum_{\eta'\in\Delta_s}\exp(m_{u,\eta'})},
\quad
\sum_{\eta\in\Delta_s}b_{u,\eta}(m)=1,
\label{eq:interp_weight}
\end{equation}
where $m=\{m_{u,\eta}\mid u\in\Omega,\eta\in\Delta_s\}$ denotes all interpolation logits. The RGB image used for guide alignment is then
\begin{equation}
\tilde{r}_{i,j}(u;\ell,m)
=
\sum_{\eta\in\Delta_s}
b_{u,\eta}(m)\,
\bar{r}_{i,j}\!\left(\pi(u+\eta);\ell\right),
\label{eq:optimized-rgb}
\end{equation}
where $\pi(\cdot)$ maps out-of-bound coordinates back to $\Omega$ according to the boundary handling rule.
The aggregation weights also define an expected proxy coordinate and a dense displacement field:
\begin{equation}
\mathbf{c}_{i,j}(u;\ell)
=
\sum_{k=1}^{K}
a_{u,k}(\ell)\,v_{u,k},
\quad
\mathbf{d}_{i,j}(u;\ell)
=
\mathbf{c}_{i,j}(u;\ell)-u.
\label{eq:soft-coordinate}
\end{equation}
Here, $\mathbf{c}_{i,j}(u;\ell)\in\mathbb{R}^{2}$ is the expected source coordinate in the proxy image, and $\mathbf{d}_{i,j}(u;\ell)\in\mathbb{R}^{2}$ is the displacement vector from the guide coordinate to the expected proxy coordinate. This representation exposes the synthesis process as a soft spatial warp and enables direct regularization of the coordinate field.

The aggregation logits and interpolation logits are estimated by aligning the synthesized RGB with the guide RGB:
\begin{equation}
(\ell^\star,m^\star)
=
\arg\min_{\ell,m}
\mathcal{L}_{\mathrm{warp}}(\ell,m).
\label{eq:warp-argmin}
\end{equation}
The objective is
\begin{equation}
\begin{aligned}
\mathcal{L}_{\mathrm{warp}}
=&\;
\lambda_{1}\|\tilde{r}_{i,j}-g_{i,j}\|_{1}
+\lambda_{p}\mathcal{L}_{\mathrm{patch}}(\tilde{r}_{i,j},g_{i,j})\\
&+\lambda_{\mathrm{mi}}\mathcal{L}_{\mathrm{MI}}(\tilde{r}_{i,j},g_{i,j})
+\lambda_{\mathrm{ssim}}\mathcal{L}_{\mathrm{SSIM}}(\tilde{r}_{i,j},g_{i,j})\\
&+\lambda_{\nabla}\|\nabla\tilde{r}_{i,j}-\nabla g_{i,j}\|_{1}
+\lambda_{s}\mathcal{L}_{\mathrm{smooth}}(\widehat{\mathbf{c}}_{i,j})\\
&+\lambda_{d}\mathcal{L}_{\mathrm{dist}} .
\end{aligned}
\label{eq:warp-loss}
\end{equation}
For compactness, $\tilde{r}_{i,j}$ in Eq.~\eqref{eq:warp-loss} denotes $\tilde{r}_{i,j}(\cdot;\ell,m)$. The normalized coordinate field is $\widehat{\mathbf{c}}_{i,j}(u)=(c_{i,j}^{y}(u)/H,c_{i,j}^{x}(u)/W)$, where $c_{i,j}^{y}(u)$ and $c_{i,j}^{x}(u)$ are the vertical and horizontal components of $\mathbf{c}_{i,j}(u)$. The term $\mathcal{L}_{\mathrm{patch}}$ is an $\ell_1$ loss between unfolded RGB patches, $\mathcal{L}_{\mathrm{MI}}$ is a differentiable mutual-information loss, $\mathcal{L}_{\mathrm{SSIM}}=1-\mathrm{SSIM}$ is the structural-similarity loss, and $\mathcal{L}_{\mathrm{smooth}}$ penalizes spatial variation of the normalized coordinate field. The descriptor-distance prior is defined as
\begin{equation}
\begin{aligned}
\mathcal{L}_{\mathrm{dist}}
=&
\frac{1}{|\Omega|}
\sum_{u\in\Omega}
\sum_{k=1}^{K}
a_{u,k}(\ell)\,\widehat{d}_{u,k},
\\
\widehat{d}_{u,k}
=&
\frac{d_{u,k}}
{\frac{1}{|\Omega|K}\sum_{u'\in\Omega}\sum_{k'=1}^{K}d_{u',k'}+\epsilon}.
\label{eq:distance-prior}
\end{aligned}
\end{equation}
Here, $|\Omega|$ is the number of pixels, $\widehat{d}_{u,k}$ is the normalized descriptor distance, and $\epsilon$ prevents division by zero. The coefficients $\lambda_{1}$, $\lambda_{p}$, $\lambda_{\mathrm{mi}}$, $\lambda_{\mathrm{ssim}}$, $\lambda_{\nabla}$, $\lambda_s$, and $\lambda_d$ balance the RGB fidelity, patch consistency, mutual-information alignment, structural similarity, gradient preservation, coordinate smoothness, and candidate-distance prior terms, respectively.

\textbf{Step3: Shared Spectral Transfer and Local Interpolation.}
After estimating the RGB-conditioned transport parameters, we fix the optimized aggregation and interpolation weights as
\begin{equation}
a^{\star}_{u,v}=a_{u,v}(\ell^{\star}),\quad v\in\mathcal{C}_{i,j}(u),
\quad
b^{\star}_{u,\eta}=b_{u,\eta}(m^{\star}),\quad \eta\in\Delta_s .
\label{eq:fixed_weights}
\end{equation}
Here, $a^{\star}_{u,v}$ is the optimized aggregation weight assigned to candidate proxy location $v$ for guide pixel $u$, and $b^{\star}_{u,\eta}$ is the optimized local interpolation weight for spatial offset $\eta$.

The proxy HSI is then transferred using exactly the same aggregation and interpolation weights estimated from RGB alignment. The pre-interpolation HSI is
\begin{equation}
\bar{y}_{i,j}(u)
=
\sum_{v\in\mathcal{C}_{i,j}(u)}
a^{\star}_{u,v}\,p_i(v),
\label{eq:pre_interp_hsi}
\end{equation}
where $p_i(v)\in\mathbb{R}^{B}$ is the full spectral vector of the proxy HSI $p_i$ at candidate location $v$, and $B$ is the number of spectral bands. The final synthesized HSI is
\begin{equation}
\tilde{y}_{i,j}(u)
=
\sum_{\eta\in\Delta_s}
b^{\star}_{u,\eta}\,
\bar{y}_{i,j}(\pi(u+\eta)).
\label{eq:final_hsi}
\end{equation}

This shared-operator design is essential for spectral consistency. At each guide location, all spectral bands are transported using the same candidate weights $a_{u,k}^{\star}$ and the same local interpolation weights $b_{u,\eta}^{\star}$. Therefore, the method does not assign different spatial correspondences to different spectral channels. Once a set of proxy pixels has been selected for a guide location, the entire spectral vector is transferred through a common spatial rule, preserving the proxy-derived spectral structure while adapting the spatial layout to the guide RGB.

\textbf{Discussion.}
Although the implementation of the warp-based generation contains multiple steps, including sparse candidate aggregation and shared local interpolation, these operations can be equivalently composed into a single sparse stochastic warp. Specifically, if $A_{i,j}$ denotes the soft aggregation matrix and $B_{i,j}$ denotes the shared local interpolation matrix for the proxy-guide pair $(i,j)$, then the overall transport can be written as
\begin{equation}
T_{i,j}=B_{i,j}A_{i,j},
\label{eq:one_step_warp_discussion}
\end{equation}
and the transferred RGB and HSI satisfy
\begin{equation}
\tilde r_{i,j}=T_{i,j}r_i,
\quad
\tilde y_{i,j}=T_{i,j}p_i.
\label{eq:shared_warp_discussion}
\end{equation}
Thus, the generated HSI remains interpretable as a one-step sparse stochastic transport of proxy spectra. The detailed proof of this operator-level equivalence, including the nonnegativity, row-stochasticity, and sparsity of $T_{i,j}$, is provided in the supplementary material.

The qualitative results are visualized in Fig.~\ref{fig:points} and Fig.~\ref{fig:ge}, in which it can be seen that the improved unCLIP model produces spatially plausible target-aligned RGBs, while the proposed spectral warp preserves spectral consistency. 

\begin{figure}[t]
\centering
\includegraphics[width=\linewidth]{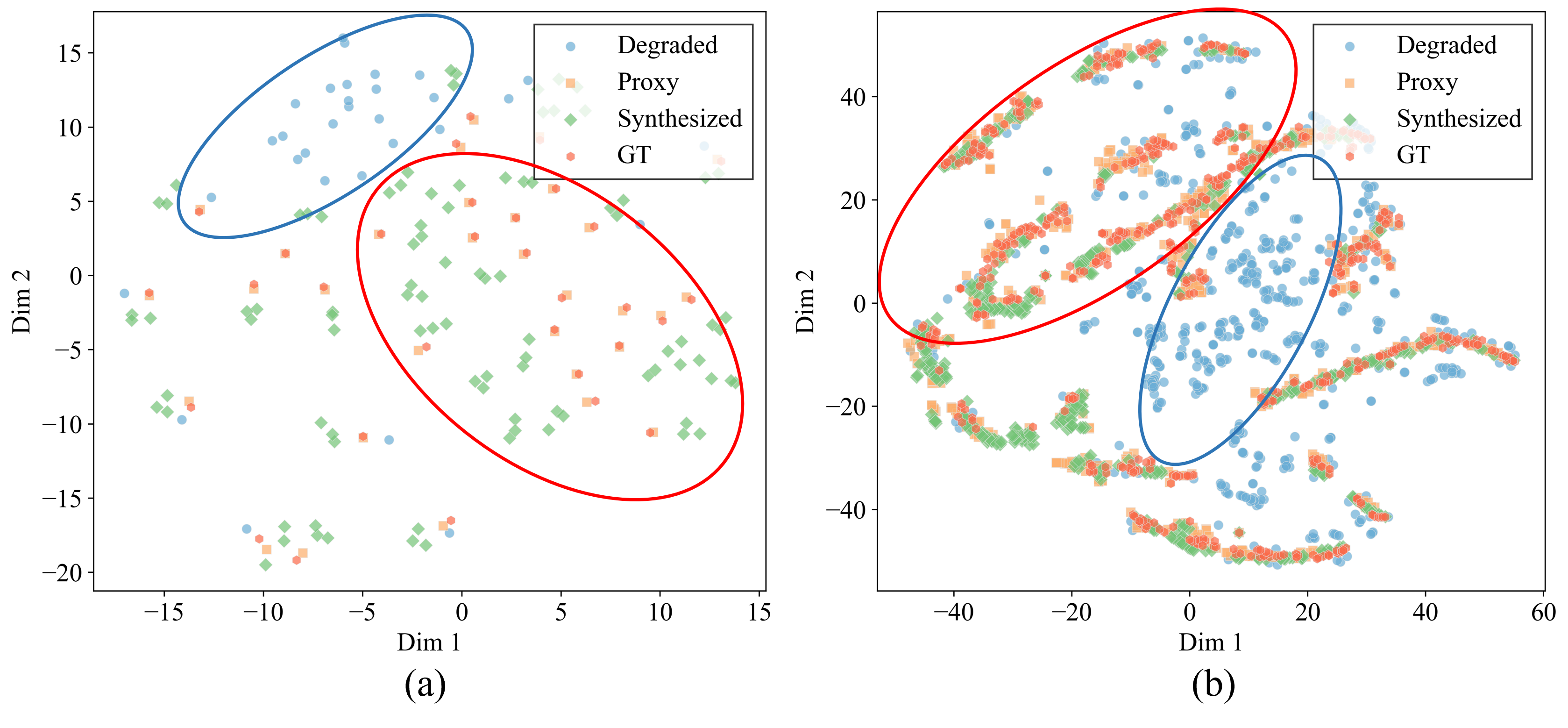}
\caption{t-SNE visualizations of \textbf{(a)} CLIP features derived from RGB images and \textbf{(b)} spectra for degraded, proxy, synthesized, and GT samples. The synthesized samples move toward the target distribution in both spatial and spectral spaces while preserving diversity.}
\label{fig:points}
\end{figure}

\begin{figure}[t]
\centering
\includegraphics[width=1\linewidth]{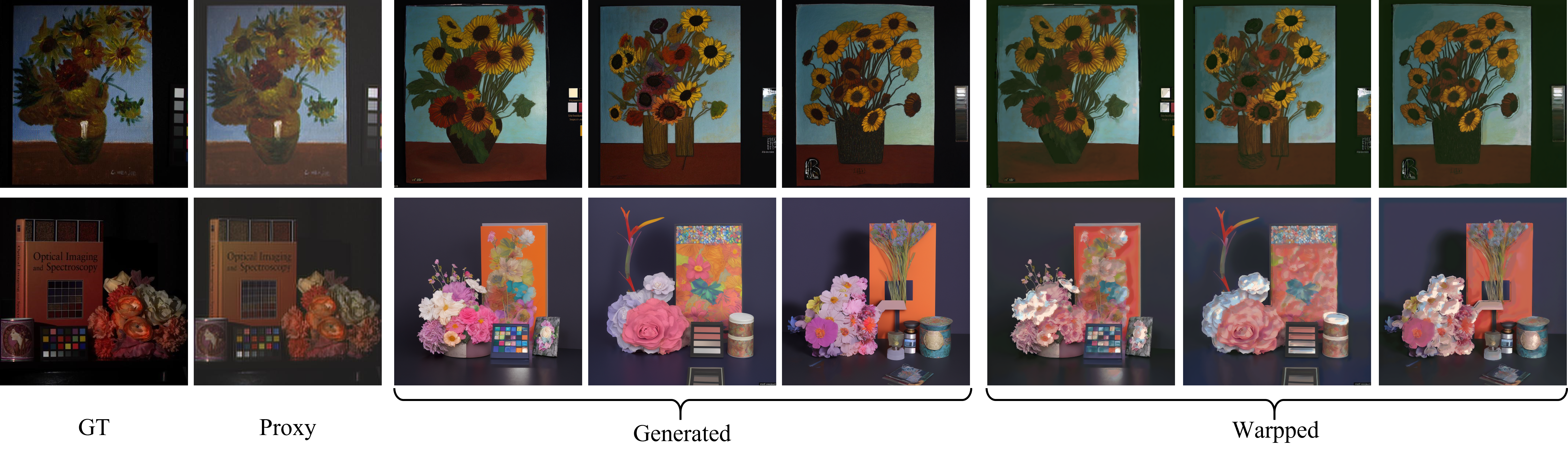}
\caption{Examples of GT images, proxy images, and synthesized samples. The synthesized RGBs retain the scene semantics of the proxies while introducing target-aligned appearance variations.}
\label{fig:ge}
\end{figure}

\subsection{Aligned Supervised Finetuning}

After proxy and synthesized HSIs are obtained, we convert them into supervised training pairs for adapting the downstream restoration backbone. Let $D_{\mathrm{task}}(\cdot)$ denote the task-specific degradation operator, such as the degradation process for denoising or super-resolution. The proxy branch uses the proxy HSI $p_i$ as a pseudo-clean target and constructs its corresponding degraded input by applying $D_{\mathrm{task}}$:
\begin{equation}
\mathcal{T}_{p}
=
\left\{
\left(
D_{\mathrm{task}}(p_i),p_i
\right)
\right\}_{i=1}^{N}.
\label{eq:proxy_pairs}
\end{equation}
The synthesized branch follows the same supervised format:
\begin{equation}
\mathcal{T}_{g}
=
\left\{
\left(
D_{\mathrm{task}}(\tilde{y}_{i,j}),\tilde{y}_{i,j}
\right)
\right\}_{i=1,\ldots,N;\,j=1,\ldots,M}.
\label{eq:gen_pairs}
\end{equation}
Here, $N$ is the number of target observations, $M$ is the number of synthesized HSIs generated from each proxy, and $\tilde{y}_{i,j}$ is the synthesized HSI obtained from the warp-based generation module. This construction ensures that each degraded input is paired with the HSI from which it is derived, yielding geometrically and spectrally aligned supervision.

Let $\Phi_{\theta}$ denote the restoration network to be adapted, initialized from the source-pretrained model. The finetuning set is formed by combining the proxy and synthesized branches,
\begin{equation}
\mathcal{T}_{\mathrm{align}}
=
\mathcal{T}_{p}\cup\mathcal{T}_{g}.
\label{eq:aligned_set}
\end{equation}
The target-adapted parameters are obtained by minimizing the supervised restoration objective
\begin{equation}
\theta^{\star}
=
\arg\min_{\theta}
\frac{1}{|\mathcal{T}_{\mathrm{align}}|}
\sum_{(x,y)\in\mathcal{T}_{\mathrm{align}}}
\mathcal{L}_{\mathrm{sup}}
\left(
\Phi_{\theta}(x),y
\right),
\label{eq:aligned_finetune}
\end{equation}
where $(x,y)$ denotes a degraded-clean training pair, and $\mathcal{L}_{\mathrm{sup}}$ is the supervised loss used by the corresponding restoration backbone.



\subsection{Theoretical Analysis}
\label{sec:theoretical_insight}

We provide a conditional analysis to explain why semantic-guided augmentation improves target-domain restoration. 
Let $Y\sim\mu_t$ denote a clean HSI from the full target distribution, and let $Y_o\sim\mu_o$ denote a clean HSI from the observed target subset used for adaptation. The full target test distribution is
\begin{equation}
P_t:=\mathcal L\bigl(D_{\mathrm{task}}(Y),Y\bigr).
\label{eq:main_target_dist}
\end{equation}
For an observed scene, write the proxy HSI as
\begin{equation}
P_o=Y_o+E,
\label{eq:main_proxy_error}
\end{equation}
where $E$ is the proxy error. The clean counterpart and the actual proxy branch are
\begin{equation}
\begin{aligned}
P_p^{\star}&:=\mathcal L\bigl(D_{\mathrm{task}}(Y_o),Y_o\bigr),\\
P_p&:=\mathcal L\bigl(D_{\mathrm{task}}(P_o),P_o\bigr).
\end{aligned}
\label{eq:main_proxy_dists}
\end{equation}
For the generated branch, the same realized semantic warp $T$ is frozen when comparing the actual generated pair with its clean counterpart:
\begin{equation}
\begin{aligned}
P_g^{\star}&:=\mathcal L\bigl(D_{\mathrm{task}}(TY_o),TY_o\bigr),\\
P_g&:=\mathcal L\bigl(D_{\mathrm{task}}(TP_o),TP_o\bigr).
\end{aligned}
\label{eq:main_generated_dists}
\end{equation}
This construction gives the exact generated-label error identity $TP_o-TY_o=TE$.

Let $\rho$ denote an ideal family of admissible target-aligned semantic warps, and define the corresponding vicinal distribution
\begin{equation}
P_v:=\mathcal L\bigl(D_{\mathrm{task}}(\widetilde T Y_o),\widetilde T Y_o\bigr),
\quad \widetilde T\sim\rho .
\label{eq:main_vicinal_dist}
\end{equation}
The coverage and generator-mismatch quantities are
\begin{equation}
\begin{aligned}
\Delta_p&:=\Wass(P_t,P_p^{\star}),\\
\Delta_w&:=\Wass(P_t,P_v),\\
e_w&:=\Wass(P_g^{\star},P_v),
\end{aligned}
\label{eq:main_delta_terms}
\end{equation}
where $\Wass(\cdot,\cdot)$ denotes the 1-Wasserstein distance under the pair metric $d((x,y),(x^{\prime},y^{\prime}))=\|x-x^{\prime}\|_F+\|y-y^{\prime}\|_F$. Here $\Delta_p$ measures the support gap of the observed proxy scenes, $\Delta_w$ measures the support gap of the ideal semantic-warp distribution, and $e_w$ measures how far the realized warp is from that ideal vicinal family. 

For a mixture ratio $\alpha\in[0,1]$, where $\alpha$ is the fraction of generated samples in finetuning, define the clean and actual mixed training distributions as
\begin{equation}
P_{\alpha}^{\star}:=(1-\alpha)P_p^{\star}+\alpha P_g^{\star},
\quad
P_{\alpha}:=(1-\alpha)P_p+\alpha P_g.
\label{eq:main_mixture_dist}
\end{equation}
The clean coverage term and actual-vs-clean perturbation term are
\begin{equation}
\begin{aligned}
\Delta_{\alpha}&:=(1-\alpha)\Delta_p+
\alpha\Delta_w+\alpha e_w,\\
\varepsilon(\alpha)&:=(1-\alpha)\varepsilon_p+
\alpha\varepsilon_g,
\end{aligned}
\label{eq:main_delta_epsilon_alpha}
\end{equation}
with
\begin{equation}
\varepsilon_p:=\Wass(P_p,P_p^{\star}),
\quad
\varepsilon_g:=\Wass(P_g,P_g^{\star}).
\label{eq:main_pair_perturb_terms}
\end{equation}

For any pair distribution $Q$, let $R_Q(f)$ be the expected supervised risk over $Q$, and let $R_t(f):=R_{P_t}(f)$.  Let $F^{\star}$ be the Bayes restorer under $P_t$, and let $f_{\mathcal F}^{\star}$ be a fixed comparator in the chosen hypothesis class $\mathcal F$. Define the risk-level bias terms as
\begin{equation}
\begin{aligned}
B_0^2
&:=\bigl[R_{P_p^{\star}}(f_{\mathcal F}^{\star})-R_t(F^{\star})\bigr]_+,\\
B_{\mathrm{warp}}^2
&:=\bigl[R_{P_g^{\star}}(f_{\mathcal F}^{\star})-R_{P_p^{\star}}(f_{\mathcal F}^{\star})\bigr]_+,
\end{aligned}
\label{eq:main_bias_terms}
\end{equation}
where $[a]_+:=\max\{a,0\}$. The first term is the comparator bias on the clean proxy branch, and the second term is the additional positive risk allowance caused by replacing clean observed proxy scenes with clean warped scenes in the generated branch.

\begin{theorem}[Conditional target-risk bound for HIR-ALIGN]
\label{thm:main_bound}
Assume that $D_{\mathrm{task}}$ is $L_D$-Lipschitz. Also assume that for every $f\in\mathcal F$,
\begin{equation}
|R_P(f)-R_Q(f)|\le L\Wass(P,Q)
\label{eq:main_lipschitz_risk}
\end{equation}
for any two pair distributions $P$ and $Q$. If finetuning on the actual mixture distribution returns $f_{\alpha}$ satisfying the effective oracle-type condition
\begin{equation}
R_{P_{\alpha}}(f_{\alpha})
\le
\inf_{f\in\mathcal F} R_{P_{\alpha}}(f)
+\frac{C_v}{m_{\mathrm{eff}}(\alpha)},
\label{eq:main_effective_oracle}
\end{equation}
then
\begin{equation}
\begin{aligned}
R_t(f_{\alpha})
\le\;&
R_t(F^{\star})+B_0^2+
\alpha B_{\mathrm{warp}}^2\\
&+\frac{C_v}{m_{\mathrm{eff}}(\alpha)}
+L\Delta_{\alpha}+2L\varepsilon(\alpha).
\end{aligned}
\label{eq:main-risk-bound}
\end{equation}
$C_v$ is an effective-view complexity constant. Moreover, if the realized warp satisfies $T^{\top}T\preceq\kappa I$ almost surely, then
\begin{equation}
\begin{aligned}
\varepsilon_p
&\le (1+L_D)\,\mathbb{E}\|E\|_F,\\
\varepsilon_g
&\le (1+L_D)
\sqrt{\kappa\,\mathbb{E}\|E\|_F^2}.
\end{aligned}
\label{eq:main_pair_perturb_bound}
\end{equation}
\end{theorem}

Theorem~\ref{thm:main_bound} is the key inequality used to interpret HIR-ALIGN. It separates the target risk into five interpretable quantities: the comparator bias $B_0^2$, the warp bias $\alpha B_{\mathrm{warp}}^2$, the effective-view term $C_v/m_{\mathrm{eff}}(\alpha)$, the clean coverage term $L\Delta_{\alpha}$, and the actual pair-perturbation term $2L\varepsilon(\alpha)$. The supplementary material contains the complete proof, including the Wasserstein mixture bound, the actual-vs-clean perturbation bound, and the empirical-risk argument behind Eq.~\eqref{eq:main_effective_oracle}.

\textbf{When the mixed distribution improves over sourcedomain training. }
Let $P_s$ denote the source-domain training distribution and define
\begin{equation}
\Delta_s:=\Wass(P_t,P_s).
\label{eq:main_source_gap}
\end{equation}
Assuming that the source-trained predictor $f_s$ satisfies the analogous source-domain effective-oracle condition $R_{P_s}(f_s)\le \inf_{f\in\mathcal F}R_{P_s}(f)+C_s/m_s$, the same argument as in Theorem~\ref{thm:main_bound} yields
\begin{equation}
R_t(f_s)
\le
R_t(F^{\star})
+B_{\mathrm{src}}^2
+\frac{C_s}{m_s}
+L\Delta_s,
\label{eq:main_source_bound}
\end{equation}
where $B_{\mathrm{src}}^2:=\bigl[R_{P_s}(f_{\mathcal F}^{\star})-R_t(F^{\star})\bigr]_+$ is the comparator bias associated with the source-trained hypothesis, $m_s$ is the effective source training size, and $L\Delta_s$ is the direct source-to-target distribution-shift penalty. 


The bound suggests that HIR-ALIGN can improve over source-domain training when the target-aligned gap $\Delta_{\alpha}$ is smaller than the source-to-target gap $\Delta_s$ by enough to offset the generated-branch bias and perturbation terms. Subtracting Eq.~\eqref{eq:main-risk-bound} from Eq.~\eqref{eq:main_source_bound} shows that the HIR-ALIGN upper bound is tighter whenever

\begin{equation}
\begin{aligned}
&L(\Delta_s - \Delta_{\alpha})
+\left(\frac{C_s}{m_s}
-\frac{C_v}{m_{\mathrm{eff}}(\alpha)}\right)
+\bigl(B_{\mathrm{src}}^2-B_0^2\bigr) \\
&\qquad >
\alpha B_{\mathrm{warp}}^2 + 2L\varepsilon(\alpha).
\end{aligned}
\label{eq:improvement_condition_full}
\end{equation}

Under the additional bias-ordering condition $B_{\mathrm{src}}^2\ge B_0^2$, this reduces to the simpler sufficient condition
\begin{equation}
L(\Delta_s - \Delta_{\alpha}) + \left(\frac{C_s}{m_s} - \frac{C_v}{m_{\mathrm{eff}}(\alpha)}\right) > \alpha B_{\mathrm{warp}}^2 + 2L\varepsilon(\alpha).
\label{eq:improvement_condition}
\end{equation}
Our algorithmic design is intended to make this sufficient condition easier to satisfy through three mechanisms:
\begin{itemize}
    \item \textbf{Bridging the Gap:} The \textit{semantic-guided branch} uses diffusion-based synthesis to create diverse target-compatible layouts. This explicitly expands the training support in the target domain, causing $L(\Delta_s - \Delta_{\alpha})$ to become a large positive gain that dominates the bound.
    \item \textbf{Increasing the Effective Number of Views:}
    The generated samples provide additional target-aligned training views around the observed proxy scenes, which enlarges the effective training support and increase $m_{\mathrm{eff}}(\alpha)$.    
    \item \textbf{Bounding the Warp Bias:} While generation introduces $B_{\mathrm{warp}}^2$, our \textit{sparse spectral warp} constrains the transfer operator $T_{i,j}$ to be a sparse, row-stochastic combination of observed spectra. This yields convex-combination spectral transfer, and together with the explicit overlap-control assumption $T^{\top}T\preceq\kappa I$ used in Eq.~\eqref{eq:main_pair_perturb_bound}, it keeps the inherited perturbation term controlled.
\end{itemize}

\textbf{When the mixed distribution improves over using only proxy or generated samples.}
For proxy-only, generated-only, and mixed finetuning, the nonconstant parts of Eq.~\eqref{eq:main-risk-bound} are
\begin{align}
U_{\mathrm{proxy}}
&:=
B_0^2+\frac{C_v}{m_{\mathrm{eff}}(0)}
+L\Delta_p+2L\varepsilon_p,
\label{eq:main_u_proxy}\\
U_{\mathrm{gen}}
&:=
B_0^2+B_{\mathrm{warp}}^2
+\frac{C_v}{m_{\mathrm{eff}}(1)}
+L(\Delta_w+e_w)+2L\varepsilon_g,
\label{eq:main_u_gen}\\
U_{\mathrm{mix}}(\alpha)
&:=
B_0^2+\alpha B_{\mathrm{warp}}^2
+\frac{C_v}{m_{\mathrm{eff}}(\alpha)}
+L\Delta_{\alpha}+2L\varepsilon(\alpha).
\label{eq:main_u_mix}
\end{align}

The proxy-only branch has stable supervision because it is directly derived from observed target scenes. Its perturbation term is controlled by $\varepsilon_p$, but its coverage remains limited by $\Delta_p$ because it cannot create target-domain views beyond the observed samples. The generated-only branch improves coverage by replacing $\Delta_p$ with $\Delta_w+e_w$ and provides more diverse target-compatible views, but it fully depends on generated samples and therefore fully carries $B_{\mathrm{warp}}^2$ and $\varepsilon_g$. The mixed distribution keeps the stable proxy anchor while adding generated samples only as an $\alpha$-weighted branch. As a result, the coverage term becomes the interpolation $\Delta_{\alpha}$, the perturbation becomes the interpolation $\varepsilon(\alpha)$, and the warp bias is reduced from $B_{\mathrm{warp}}^2$ to $\alpha B_{\mathrm{warp}}^2$. This is precisely why HIR-ALIGN uses a proxy-plus-generated mixture, as proxy samples reduce the risk of drifting away from the observed target content, while semantic-guided generated samples enlarge the target support and increase $m_{\mathrm{eff}}(\alpha)$.

\paragraph{Why semantic-guided augmentation is preferable to random shuffling.}
Let \(\Delta_{\mathrm{shuf}}\) denote the clean coverage gap induced by a
randomly shuffled counterpart. Although a shuffle/permutation may preserve the
proxy-error norm, it breaks material layouts, object boundaries, and local
spatial coherence. Thus the shuffled samples are typically off the target clean
manifold and lead to a much larger coverage gap, i.e.,
\[
\Delta_{\mathrm{shuf}} \gg \Delta_w+e_w .
\]
Therefore, the semantic-guided warp gives a smaller coverage term.
Moreover, random shuffling should contribute little to \(m_{\mathrm{eff}}\),
because \(m_{\mathrm{eff}}\) counts useful target-compatible views rather than
arbitrary permutations. In contrast, semantic-guided samples expand the target
support with plausible layouts while keeping spectra anchored to the proxy HSIs.


\textbf{Remark.}
The bound shows that the proposed augmentation introduces controllable perturbations from proxy and warp errors, while improving target-domain coverage through additional target-aligned views.
\textbf{Overall, HIR-ALIGN improves restoration performance because the in-distribution gain brought by better target coverage outweighs the controlled errors introduced by semantic-guided augmentation.}

\section{Experiments}

\subsection{Datasets}

The main simulated experiments use CAVE \cite{CAVE} and KAIST \cite{KAIST}. CAVE contains 32 HSIs of spatial size $512\times512$ with 31 spectral bands spanning visible wavelengths from 400\,nm to 700\,nm. KAIST contains 30 HSIs with 31 bands spanning 420\,nm to 720\,nm; following common practice, we resize each image to $512\times512$ for training and evaluation. These two datasets are well suited to our study because they represent natural-scene hyperspectral domains with different scene statistics and acquisition characteristics, thus forming a meaningful source-to-target adaptation benchmark. We further report real-world denoising results on HSIDwrD \cite{zhang2021hyperspectral} in Sec.~\ref{sec:real_world}. For the main simulated experiments, we adopt a transductive
target-adaptation setting. All degraded target observations from CAVE or KAIST are used to construct proxy and synthesized samples for finetuning, while the corresponding clean HSIs are strictly withheld and used only for metric computation. Thus, the target datasets are not divided into conventional supervised training and testing splits in the main experiments. This protocol matches the practical scenario where degraded target observations are available but clean target references are unavailable. We further evaluate generalization to unseen target images using an observed/unobserved split in ablation study.

\subsection{Compared Methods}
We compare HIR-ALIGN with source-only supervised restorers and unsupervised or optimization-based restoration methods. For all supervised backbones, each model is first trained on the source dataset ICVL \cite{ICVL}. 
We denote source-trained backbone ''Src'', and denote the model fine-tuned with the proposed HIR-ALIGN data as ''Ours''. To evaluate the effectiveness of the generated target-aligned data and avoid performance differences caused by unequal optimization steps, the ''Src'' baseline is also fine-tuned with the same schedule and the same number of epochs as ''Ours''. Specifically, ''Src'' is fine-tuned only with proxy data which reflects the restoration capability and bias inherited from the source-trained model, while ''Ours'' is fine-tuned with the full HIR-ALIGN training set, which contains both proxy samples and synthesized target-aligned samples, to ensure fairness.
The supervised methods include MP-HSIR \cite{MP-HSIR}, PromptIR \cite{promptir}, T3SC \cite{T3SC}, SERT \cite{SERT}, SSUMamba \cite{SSUMamba}, ESSAformer \cite{ESSAformer}, MCNet \cite{MCNet}, SFCSR \cite{SFCSR}, Bi3DQRNN \cite{Bi3D}, GDRRN \cite{GDRRN}, SwinIR \cite{liang2021swinir}, Restormer \cite{zamir2022restormer}, and PromptHSI \cite{cheng2026prompthsi}. The unsupervised, self-supervised, or optimization-based methods include LRTV \cite{LRTV}, LRTDTV \cite{LRTDTV}, Enhanced-3DTV \cite{Enhanced-3DTV}, RCTV \cite{RCTV}, HIR-Diff \cite{pang2024hir}, DIP \cite{UlyanovVL17}, DHP \cite{9022040}, DIP-DKP \cite{10658451}, PnP \cite{wang2023tuning}, UAU \cite{Tang_2023_CVPR}, HyperEI \cite{Li2024EquivariantIF}, and R-DLRHyIn \cite{10032531}. HSIGene \cite{pang2024hsigene} is additionally used in the generation-source ablation study.

\subsection{Implementation Details}

\textbf{Degradation simulation.} For denoising, we simulate both non-i.i.d. Gaussian noise (noise levels 10-70) and complex mixture noise which consisted of: 1) impulse noise on 1/3 of the bands, with intensities ranging from 10\% to 70\%; 2) stripes on 5\%-15\% of columns on 1/3 of the bands; 3) deadlines on 5\%-15\% of columns on 1/3 of the bands. For super-resolution, we use bicubic $4\times$ downsampling. For band completion, inpainting, and deblurring, the band-missing ratio, masking ratio, and blur radius are set to 0.3, 0.9, and 15, respectively. During aligned finetuning, the same degradation families are replayed on proxy and synthesized HSIs so that each restoration backbone is trained in a conventional supervised manner.

\textbf{Proxy generation.} In the first stage, degraded target HSIs are processed by ICVL-pretrained restoration models to obtain proxy HSIs. These proxies are the only target-domain supervisory labels used by our framework. Proxy RGBs are constructed from three visible proxy bands centered around 660\,nm, 550\,nm, and 470\,nm.

\textbf{Distribution-adaptive RGB synthesis.} 
We fine-tune the CLIP-vision encoder of unCLIP on COCO2017 \cite{COCO2017} using synthetic Gaussian blur levels. During RGB generation, we employ fixed positive and negative prompts to encourage detailed, clean outputs and suppress obviously blurry or noisy generations. Reverse diffusion is initialized from Gaussian-perturbed image embeddings rather than from pure noise. Each proxy sample produces three target-aligned RGB guides.

\textbf{Warp-based HSI generation.} Descriptor extraction uses $5\times5$ patches, and the gradient and chroma channels are reweighted by 0.5 and 0.25, respectively. For each guide pixel, we first retrieve 16 full-resolution appearance neighbors, expand each seed inside a $3\times3$ neighborhood, and retain 16 soft candidates after refinement. A shared $7\times7$ local interpolation kernel is then optimized jointly with the aggregation weights for 200 iterations using a learning rate of 0.05. The softmax temperature is fixed to one in all reported experiments. Importantly, the retrieval is performed directly at the original image resolution and does not use an additional explicit spatial-distance penalty in the reported setting. Therefore the effective forward path is exactly the appearance-driven sparse aggregation plus shared local interpolation described in Sec.~III-B.

\textbf{Training details.}
For source pretraining, we randomly select 100 HSIs from the ICVL dataset as source training images. SERT, T3SC, and SSUMamba are pretrained for 50 epochs using the Adam optimizer with an initial learning rate of \num{1e-4}. ESSAformer, MCNet, and SFCSR are pretrained for 80 epochs using Adam with an initial learning rate of \num{1e-4}. MP-HSIR and PromptIR are pretrained for 50 epochs using the AdamW optimizer with an initial learning rate of \num{2e-4}.

For target-adaptive finetuning, each ''Src'' and corresponding ''Ours'' model is initialized from the same ICVL-pretrained checkpoint and trained with the same optimization settings. Specifically, SERT, T3SC, and SSUMamba are finetuned for 10 epochs using Adam with an initial learning rate of \num{1e-4}, which is reduced by half after the 5th epoch. ESSAformer, MCNet, and SFCSR are finetuned for 20 epochs using Adam with an initial learning rate of \num{1e-4}, which is also reduced by half after the 10th epoch. MP-HSIR and PromptIR are finetuned for 20 epochs using AdamW with an initial learning rate of \num{2e-4}. For each backbone, the same optimizer, learning rate, number of epochs, and degradation replay strategy are used for the paired ''Src'' and ''Ours'' models; the only difference lies in the finetuning data, where ''Src'' uses proxy samples only and ''Ours'' uses the full HIR-ALIGN training set containing both proxy and synthesized target-aligned samples.

\textbf{Metrics.} We report PSNR, SSIM, and SAM throughout the paper. Since our goal is target-aligned supervised adaptation rather than one-off visual enhancement, the metrics are always computed on the restoration outputs of the final finetuned model rather than only on the intermediate synthesized data.

\subsection{Main Results}
\begin{table}[ht]
\renewcommand{\arraystretch}{1.1}
\setlength{\tabcolsep}{3pt}
  \centering
  \caption{Performance on the Gaussian denoising task on CAVE and KAIST. The best and second-best results are shown in \textbf{bold} and \underline{underlined}, respectively.}
  \resizebox{.98\columnwidth}{!}{
  \begin{tabular}{c| c c c | c c c}
    \Xhline{1.2pt}
    Dataset & \multicolumn{3}{c|}{CAVE} & \multicolumn{3}{c}{KAIST} \\
    \hline
    Method & PSNR & SSIM & SAM & PSNR & SSIM & SAM \\
    \hline
    HIR-Diff \cite{pang2024hir} & 29.17 & 0.744 & 25.11 & 26.56 & 0.632 & 26.34 \\
    \hline
    LRTV \cite{LRTV} & 31.30 & 0.854 & 9.67 & 31.19 & 0.861 & 11.79 \\
    LRTDTV \cite{LRTDTV} & 30.92 & 0.745 & 15.18 & 30.82 & 0.725 & 17.19 \\
    Enhanced-3DTV \cite{Enhanced-3DTV} & 31.48 & 0.903 & 10.87 & 31.16 & 0.880 & 14.48 \\
    RCTV \cite{RCTV} & 34.71 & 0.910 & 6.99 & 34.30 & 0.912 & 8.52 \\
    \hline
    T3SC (Src) \cite{T3SC} & 37.39 & 0.924 & 12.37 & 38.41 & 0.942 & 9.23 \\
    T3SC (Ours) & 38.08 & 0.933 & 11.46 & 39.10 & 0.948 & 8.09 \\
    \hline
    SERT (Src) \cite{SERT} & 39.94 & 0.964 & 7.67 & 40.12 & 0.965 & 6.62 \\
    SERT (Ours) & 40.57 & \underline{0.967} & \textbf{6.67} & 40.74 & 0.967 & 6.55 \\
    \hline
    PromptIR (Src) \cite{promptir} & 37.17 & 0.948 & 10.05 & 37.25 & 0.946 & 8.23 \\
    PromptIR (Ours) & 37.58 & 0.952 & 9.30 & 37.67 & 0.946 & 7.65 \\
    \hline
    SSUMamba (Src) \cite{SSUMamba} & \underline{41.65} & \textbf{0.995} & 7.64 & \underline{41.96} & \underline{0.995} & \underline{6.07} \\
    SSUMamba (Ours) & \textbf{42.36} & \textbf{0.995} & \underline{6.81} & \textbf{42.04} & \textbf{0.996} & \textbf{5.86} \\
    \hline
    MP-HSIR (Src) \cite{MP-HSIR} & 38.35 & 0.960 & 7.74 & 38.57 & 0.956 & 7.03 \\
    MP-HSIR (Ours) & 38.69 & 0.963 & 7.55 & 38.88 & 0.954 & 7.26 \\
    \Xhline{1.2pt}
  \end{tabular}
  }
  \label{tab:denoise1}
\end{table}

\begin{table}[t]
\renewcommand{\arraystretch}{1.1}
\setlength{\tabcolsep}{3pt}
  \centering
  \caption{Performance on the complex denoising task on CAVE and KAIST. The best and second-best results are shown in \textbf{bold} and \underline{underlined}, respectively.}
  \resizebox{.98\columnwidth}{!}{
  \begin{tabular}{c| c c c | c c c}
    \Xhline{1.2pt}
    Dataset & \multicolumn{3}{c|}{CAVE} & \multicolumn{3}{c}{KAIST} \\
    \hline
    Method & PSNR & SSIM & SAM & PSNR & SSIM & SAM \\
    \hline
    HIR-Diff \cite{pang2024hir} & 17.75 & 0.589 & 43.88 & 17.40 & 0.532 & 50.67 \\
    \hline
    LRTV \cite{LRTV} & 25.05 & 0.683 & 40.22 & 24.89 & 0.682 & 35.09 \\
    LRTDTV \cite{LRTDTV} & 33.32 & 0.815 & 32.03 & 33.46 & 0.811 & 27.92 \\
    Enhanced-3DTV \cite{Enhanced-3DTV} & 33.77 & 0.959 & 12.09 & 35.34 & \underline{0.965} & 10.50 \\
    RCTV \cite{RCTV} & 39.40 & 0.979 & \textbf{4.76} & 36.99 & 0.956 & \textbf{6.93} \\
    \hline
    T3SC (Src) \cite{T3SC} & 33.90 & 0.797 & 22.03 & 34.70 & 0.800 & 19.07 \\
    T3SC (Ours) & 35.13 & 0.841 & 19.61 & 35.51 & 0.823 & 19.26 \\
    \hline
    SERT (Src) \cite{SERT} & 39.46 & 0.954 & 14.07 & 39.98 & 0.935 & 8.92 \\
    SERT (Ours) & 39.96 & 0.959 & 9.76 & 40.59 & 0.958 & 8.90 \\

    \hline
    PromptIR (Src) \cite{promptir} & 35.72 & 0.940 & 15.41 & 35.93 & 0.929 & 15.70 \\
    PromptIR (Ours) & 36.13 & 0.941 & 14.79 & 36.44 & 0.934 & 13.13 \\
    \hline
    SSUMamba (Src) \cite{SSUMamba} & \underline{42.27} & \underline{0.995} & 9.48 & \underline{41.69} & \textbf{0.994} & 7.79 \\
    SSUMamba (Ours) & \textbf{42.43} & \textbf{0.996} & \underline{9.31} & \textbf{41.80} & \textbf{0.994} & \underline{7.54} \\
    \hline
    MP-HSIR (Src) \cite{MP-HSIR} & 38.65 & 0.953 & 10.36 & 38.69 & 0.945 & 8.34 \\
    MP-HSIR (Ours) & 39.02 & 0.957 & 10.60 & 39.52 & 0.954 & 7.69 \\

    \Xhline{1.2pt}
  \end{tabular}
  }
  \label{tab:denoise2}
\end{table}

\begin{table}[t]
\renewcommand{\arraystretch}{1.1}
\setlength{\tabcolsep}{3pt}
  \centering
  \caption{Performance on the super-resolution task on CAVE and KAIST. The best and second-best results are shown in \textbf{bold} and \underline{underlined}, respectively.}
  \resizebox{.98\columnwidth}{!}{
  \begin{tabular}{c| c c c | c c c}
    \Xhline{1.2pt}
    Dataset & \multicolumn{3}{c|}{CAVE} & \multicolumn{3}{c}{KAIST} \\
    \hline
    Method & PSNR & SSIM & SAM & PSNR & SSIM & SAM \\
    \hline
    Bicubic & 34.31 & 0.932 & 4.12 & 32.20 & 0.926 & 3.55 \\    
    HIR-Diff \cite{pang2024hir} & 31.09 & 0.852 & 16.71 & 30.47 & 0.849 & 11.46 \\
    DIP \cite{UlyanovVL17} & 34.73 & 0.920 & 5.81 & 33.20 & 0.920 & 4.54 \\
    DHP \cite{9022040} & 35.26 & 0.932 & 5.69 & 33.31 & 0.923 & 4.60 \\
    DIP-DKP \cite{10658451} & 24.19 & 0.800 & 7.19 & 24.53 & 0.837 & 5.60 \\
    \hline
    Bi3DQRNN \cite{Bi3D} & 35.72 & 0.946 & 6.36 & 34.60 & 0.940 & 5.63 \\
    GDRRN \cite{GDRRN} & 34.39 & 0.911 & 7.60 & 32.35 & 0.889 & 7.52 \\
    \hline
    MCNet (Src) \cite{MCNet} & 37.27 & \underline{0.956} & 3.80 & 34.92 & 0.951 & 3.32 \\
    MCNet (Ours) & \underline{37.65} & \textbf{0.958} & \underline{3.76} & \textbf{35.63} & \textbf{0.955} & \textbf{3.16} \\
    \hline
    SFCSR (Src) \cite{SFCSR} & 37.38 & 0.955 & 3.77 & 34.83 & 0.950 & 3.31 \\
    SFCSR (Ours) & \textbf{37.66} & \textbf{0.958} & \textbf{3.67} & \underline{35.08} & \underline{0.952} & \underline{3.23} \\
    \hline
    ESSAformer (Src) \cite{ESSAformer} & 31.17 & 0.871 & 15.79 & 29.31 & 0.863 & 12.57 \\
    ESSAformer (Ours) & 35.12 & 0.933 & 7.54 & 32.98 & 0.913 & 7.08 \\
    \hline
    PromptIR (Src) \cite{promptir} & 31.63 & 0.914 & 4.95 & 29.68 & 0.906 & 4.30 \\
    PromptIR (Ours) & 31.79 & 0.916 & 5.30 & 29.76 & 0.907 & 4.56 \\
    \hline
    MP-HSIR (Src) \cite{MP-HSIR} & 32.89 & 0.928 & 5.02 & 30.76 & 0.920 & 4.28 \\
    MP-HSIR (Ours) & 33.37 & 0.931 & 4.70 & 31.22 & 0.924 & 4.14 \\
    \Xhline{1.2pt}
  \end{tabular}
  }
  \label{tab:SR}
\end{table}

Tables~\ref{tab:denoise1}, \ref{tab:denoise2}, and \ref{tab:SR} report the results on Gaussian denoising, complex denoising, and super-resolution. On Gaussian denoising, HIR-ALIGN improves the PSNR of all paired supervised backbones on both CAVE and KAIST. SSUMamba (Ours) achieves the best PSNR on both datasets, while SERT (Ours) and SSUMamba (Ours) obtain the best SAM values on CAVE and KAIST, respectively. The gains over the corresponding Src rows show that adding synthesized target-aligned data is more effective than proxy-only target adaptation under the same optimization budget. In complex denoising, all paired supervised backbones obtain higher PSNR after HIR-ALIGN adaptation, and SSUMamba (Ours) gives the highest PSNR on both datasets. RCTV remains competitive in SAM, indicating that optimization-based priors can still be strong for some spectral-angle measurements. In super-resolution, the adapted MCNet and SFCSR variants achieve the best or second-best PSNR and SSIM on both datasets, while unsupervised or self-supervised methods such as DIP, DHP, and DIP-DKP provide additional references. These results show that the proposed proxy-anchored synthesis benefits both degradation removal and spatial detail recovery.

\begin{table}[t]
\renewcommand{\arraystretch}{1.1}
\setlength{\tabcolsep}{3pt}
  \centering
  \caption{Performance on the deblurring task on CAVE and KAIST. The best and second-best results are shown in \textbf{bold} and \underline{underlined}, respectively.}
  \resizebox{.98\columnwidth}{!}{
  \begin{tabular}{c| c c c | c c c}
    \Xhline{1.2pt}
    Dataset & \multicolumn{3}{c|}{CAVE} & \multicolumn{3}{c}{KAIST} \\
    \hline
    Method & PSNR & SSIM & SAM & PSNR & SSIM & SAM \\
    \hline
    HIR-Diff \cite{pang2024hir} & 28.37 & 0.782 & 19.68 & 27.70 & 0.778 & 14.10 \\
    PnP \cite{wang2023tuning} & 34.09 & 0.944 & 4.81 & 32.85 & 0.938 & 4.68 \\
    UAU \cite{Tang_2023_CVPR} & 35.32 & 0.906 & 12.57 & 33.13 & 0.894 & 13.54 \\
    \hline
    SwinIR \cite{liang2021swinir} & 42.53 & 0.981 & \textbf{3.22} & 39.25 & 0.974 & \textbf{2.93} \\
    Restormer \cite{zamir2022restormer} & 42.00 & \textbf{0.986} & 3.64 & 40.19 & \underline{0.979} & 3.51 \\
    PromptHSI \cite{cheng2026prompthsi} & 23.33 & 0.646 & 17.00 & 28.02 & 0.777 & 9.49 \\
    \hline
    PromptIR (Src) \cite{promptir} & \underline{43.78} & \underline{0.985} & 3.54 & \underline{40.90} & 0.978 & 2.98 \\
    PromptIR (Ours) & \textbf{44.32} & \textbf{0.986} & 3.47 & \textbf{41.67} & \textbf{0.980} & \underline{2.96} \\
    \hline
    MP-HSIR (Src) \cite{MP-HSIR} & 42.79 & 0.982 & 3.47 & 40.48 & 0.976 & 3.03 \\
    MP-HSIR (Ours) & 42.99 & 0.982 & \underline{3.36} & 40.73 & 0.977 & 3.03 \\

    \Xhline{1.2pt}
  \end{tabular}
  }
  \label{tab:deblur}
\end{table}

\begin{table}[t]
\renewcommand{\arraystretch}{1.1}
\setlength{\tabcolsep}{3pt}
  \centering
  \caption{Performance on the band-completion task on CAVE and KAIST. The best and second-best results are shown in \textbf{bold} and \underline{underlined}, respectively.}
  \resizebox{.98\columnwidth}{!}{
  \begin{tabular}{c| c c c | c c c}
    \Xhline{1.2pt}
    Dataset & \multicolumn{3}{c|}{CAVE} & \multicolumn{3}{c}{KAIST} \\
    \hline
    Method & PSNR & SSIM & SAM & PSNR & SSIM & SAM \\
    \hline
    HIR-Diff \cite{pang2024hir} & 27.47 & 0.674 & 38.15 & 27.42 & 0.674 & 38.16 \\
    DIP \cite{UlyanovVL17} & 34.66 & 0.756 & 32.24 & 32.52 & 0.743 & 33.99 \\
    DHP \cite{9022040} & 31.02 & 0.730 & 13.77 & 32.33 & 0.766 & 11.54 \\
    \hline
    SwinIR \cite{liang2021swinir} & 45.92 & 0.990 & 5.36 & 47.56 & \textbf{0.990} & 4.17 \\
    Restormer \cite{zamir2022restormer} & 50.16 & \underline{0.991} & \underline{4.56} & 50.57 & 0.985 & \textbf{3.66} \\
    PromptHSI \cite{cheng2026prompthsi} & 36.76 & 0.965 & 10.25 & 37.02 & 0.922 & 12.22 \\
    \hline
    PromptIR (Src) \cite{promptir} & 48.49 & \underline{0.991} & 5.37 & 48.52 & 0.987 & 5.69 \\
    PromptIR (Ours) & \underline{50.17} & \textbf{0.993} & 4.77 & 50.67 & \underline{0.989} & 4.73 \\
    \hline
    MP-HSIR (Src) \cite{MP-HSIR} & 49.61 & \underline{0.991} & 4.95 & \underline{50.78} & \textbf{0.990} & \underline{3.88} \\
    MP-HSIR (Ours) & \textbf{50.73} & \textbf{0.993} & \textbf{4.25} & \textbf{51.84} & \textbf{0.990} & 4.39 \\

    \Xhline{1.2pt}
  \end{tabular}
  }
  \label{tab:bandmiss}
\end{table}

\begin{table}[t]
\renewcommand{\arraystretch}{1.1}
\setlength{\tabcolsep}{3pt}
  \centering
  \caption{Performance on the inpainting task on CAVE and KAIST. The best and second-best results are shown in \textbf{bold} and \underline{underlined}, respectively.}
  \resizebox{.98\columnwidth}{!}{
  \begin{tabular}{c| c c c | c c c}
    \Xhline{1.2pt}
    Dataset & \multicolumn{3}{c|}{CAVE} & \multicolumn{3}{c}{KAIST} \\
    \hline
    Method & PSNR & SSIM & SAM & PSNR & SSIM & SAM \\
    \hline
    HIR-Diff \cite{pang2024hir} & 23.77 & 0.715 & 37.27 & 24.31 & 0.748 & 26.36 \\
    DHP \cite{9022040} & 34.87 & 0.923 & 6.40 & 32.97 & 0.921 & \underline{4.80} \\
    HyperEI \cite{Li2024EquivariantIF} & 15.47 & 0.297 & 30.06 & 17.04 & 0.378 & 30.54 \\
    R-DLRHyIn \cite{10032531} & 38.83 & 0.956 & \underline{5.79} & 39.07 & 0.955 & 5.27 \\
    \hline
    SwinIR \cite{liang2021swinir} & \underline{40.10} & \underline{0.969} & 7.78 & 39.35 & 0.961 & 6.58 \\
    Restormer \cite{zamir2022restormer} & 35.76 & 0.943 & 9.13 & 36.34 & 0.944 & 7.08 \\
    PromptHSI \cite{cheng2026prompthsi} & 31.85 & 0.899 & 18.97 & 33.40 & 0.899 & 14.66 \\
    \hline
    PromptIR (Src) \cite{promptir} & 35.59 & 0.939 & 9.16 & 36.36 & 0.935 & 7.72 \\
    PromptIR (Ours) & 36.62 & 0.945 & 8.28 & 37.42 & 0.936 & 7.36 \\
    \hline
    MP-HSIR (Src) \cite{MP-HSIR} & 39.97 & 0.967 & 6.14 & \underline{39.83} & \underline{0.964} & 5.38 \\
    MP-HSIR (Ours) & \textbf{40.85} & \textbf{0.973} & \textbf{5.77} & \textbf{40.64} & \textbf{0.970} & \textbf{4.78} \\

    \Xhline{1.2pt}
  \end{tabular}
  }
  \label{tab:inpainting}
\end{table}

Tables~\ref{tab:deblur}, \ref{tab:bandmiss}, and \ref{tab:inpainting} report the results on deblurring, band completion, and inpainting. For deblurring, PromptIR (Ours) achieves the best PSNR and SSIM on both datasets, and MP-HSIR also benefits from HIR-ALIGN in PSNR. For band completion, MP-HSIR (Ours) obtains the best PSNR on both CAVE and KAIST, and PromptIR also shows clear improvements over its corresponding Src baseline. For inpainting, MP-HSIR (Ours) achieves the best PSNR, SSIM, and SAM on both datasets. Across these three tasks, HIR-ALIGN consistently improves the paired MP-HSIR and PromptIR rows in PSNR relative to proxy-only target-adaptive finetuning. Overall, the paired comparisons show that the proposed data-construction strategy improves general-purpose HSI restoration backbones across multiple degradation families.

Another noteworthy pattern is the consistency across datasets. CAVE and KAIST differ in content, acquisition conditions, and the severity of the source-to-target mismatch, yet the direction of the PSNR gain is unchanged for the adapted backbones. This robustness suggests that HIR-ALIGN is not exploiting a benchmark-specific artifact. Instead, it captures a broadly useful principle: when a restorer trained on a source corpus is exposed to target-domain structure through proxy-anchored synthesis, the resulting supervised fine-tuning becomes substantially more reliable.

\begin{figure*}[t]
\centering
\includegraphics[width=0.8\linewidth]{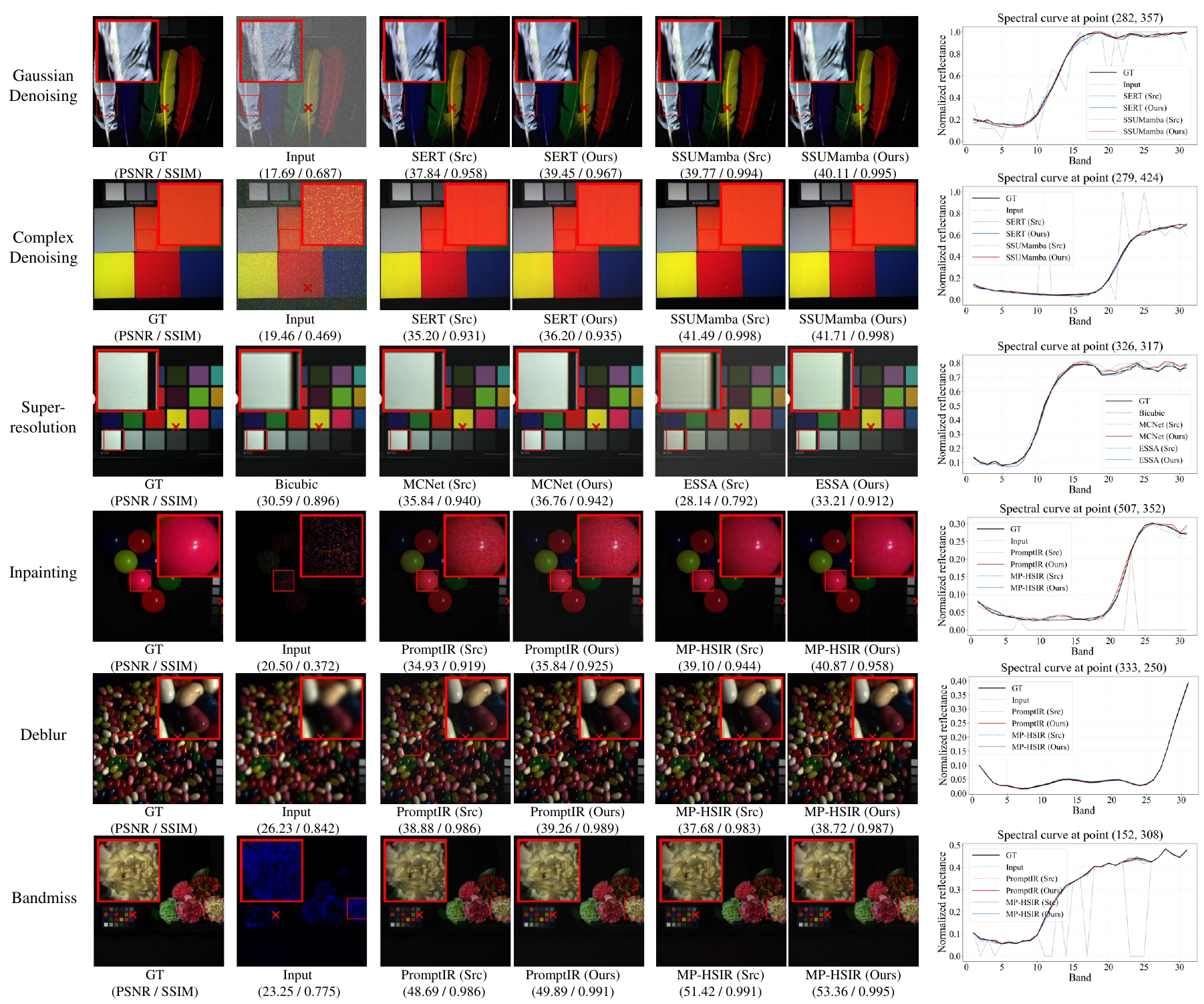}
\caption{Visual comparison between "Src" labeled baselines and the corresponding HIR-ALIGN-adapted models across six restoration tasks.}
\label{fig:comp}
\end{figure*}

The visual comparison is provided in Fig.~\ref{fig:comp}. Baselines without HIR-ALIGN-generated augmentation often suffer from residual noise, oversmoothing, incorrect high-frequency textures, or incomplete structure recovery. By contrast, HIR-ALIGN tends to produce cleaner fine details and more coherent object boundaries. The qualitative improvement is particularly noticeable when the proxy generation stage has already captured the dominant target semantics but the final baseline restorer still lacks in-domain training examples; in such cases the synthesized data effectively fill in the missing target-aligned variations that the backbone needs for finetuning.

\subsection{Experiment on Real-World Datasets}
\label{sec:real_world}

In this section, we conduct denoising experiments on the real-world dataset HSIDwrD \cite{zhang2021hyperspectral}. HSIDwrD consists of 59 hyperspectral images, each with dimensions $34 \times 696 \times 520$. In the experiment, we remove three anomalous channels (0, 11, and 22) that are visibly corrupted in the ground-truth (GT) data and crop the images to $512 \times 512$.

We evaluate HIR-ALIGN on three denoising backbones: T3SC \cite{T3SC}, SERT \cite{SERT}, and SSUMamba \cite{SSUMamba}. Complex noise is used as the degradation model for constructing the adaptation pairs. The “Src” rows denote models fine-tuned with proxy data only, while the “Ours” rows denote models finetuned with the full HIR-ALIGN training set containing both proxy data and synthesized target-aligned data. The performance of these models and several comparative methods is shown in Table~\ref{tab:real}.

Due to the distribution gap between ICVL and HSIDwrD, the proxy-only ``Src'' baselines still have limited adaptation ability on this real-world dataset. By introducing synthesized target-aligned samples, HIR-ALIGN provides additional target-domain variations for fine-tuning. As shown in Table~\ref{tab:real}, HIR-ALIGN improves T3SC, SERT, and SSUMamba over their corresponding ``Src'' baselines, indicating that the proposed data construction strategy is also effective for real-world denoising.


\begin{table}[t]
\renewcommand{\arraystretch}{1.1}
\setlength{\tabcolsep}{3pt}
  \centering
  \caption{Denoising performance on HSIDwrD. We compare Src baselines with the corresponding HIR-ALIGN-adapted models. The best and second-best results are shown in \textbf{bold} and \underline{underlined}, respectively.}
  \resizebox{0.6\columnwidth}{!}{
  \begin{tabular}{c| c c c}
    \Xhline{1.2pt}
    Method & PSNR & SSIM & SAM  \\
    \hline
    LRTV \cite{LRTV}                    & 23.81 & 0.780 & 7.72 \\
    LRTDTV \cite{LRTDTV}                & 24.19 & 0.811 & 7.94 \\
    Enhanced-3DTV \cite{Enhanced-3DTV}  & 24.24 & 0.819 & 7.75 \\
    RCTV \cite{RCTV}                    & 24.09 & 0.807 & 8.11 \\
    \hline
    T3SC (Src) \cite{T3SC}             & 24.69 & \underline{0.862} & \underline{4.10} \\
    T3SC (Ours)                         & \underline{24.71} & \textbf{0.863} & \textbf{4.03} \\
    \hline
    SERT (Src) \cite{SERT}             & 24.67 & 0.827 & 4.53 \\
    SERT (Ours)                         & \textbf{24.72} & 0.831 & 4.46 \\
    \hline
    SSUMamba (Src) \cite{SSUMamba}     & 19.89 & 0.767 & 12.37 \\
    SSUMamba (Ours)                     & 20.10 & 0.776 & 9.67 \\
    \Xhline{1.2pt}
  \end{tabular}
  }
  \label{tab:real}
\end{table}

\subsection{Ablation Study}
We conduct ablation experiments to examine the main design choices of HIR-ALIGN, including the composition of the adapted training set, the quality of the proxy spectra, the amount of synthesized data, the RGB synthesis module, the choice of generation source, generalization to unobserved target images, and the remaining gap to clean-target finetuning.

\subsubsection{Comparison of Different Training Data}
\begin{table}[t]
\setlength{\tabcolsep}{3pt}
  \centering
  \caption{Ablation study on the effect of different training data combinations.}
  \resizebox{.8\columnwidth}{!}{
  \begin{tabular}{c c | c c c | c c c}
   \Xhline{1.2pt}
   \multicolumn{2}{c|}{Training set} & \multicolumn{3}{c|}{Denoising} & \multicolumn{3}{c}{Super-resolution} \\
   \hline
   Proxy & Generated & PSNR & SSIM & SAM & PSNR & SSIM & SAM\\
   \hline
   $\checkmark$ &            & 37.39&0.924&12.37 & 37.27 &0.956 &3.80 \\
                & $\checkmark$ & 38.04 & 0.922 & 11.75 & 37.34 & 0.952 & 3.83\\
   $\checkmark$ & $\checkmark$ &\textbf{38.08}&\textbf{0.933}&\textbf{11.46} & \textbf{37.65}&\textbf{0.958}&\textbf{3.76}\\
   \Xhline{1.2pt}
  \end{tabular}
  }
  \label{tab:ab1}
\end{table}

We first evaluate the contribution of the proxy branch and the generated branch on CAVE. Specifically, we conduct Gaussian denoising with T3SC and super-resolution with MCNet, and fine-tune the restorers using proxy samples only, generated samples only, or their combination. The results are reported in Table~\ref{tab:ab1}.

The results show that using both proxy and generated samples gives the best overall performance. Proxy samples provide stable target-domain supervision because they are directly derived from degraded target observations, while generated samples introduce additional target-aligned variations. Using only one branch leaves one of these advantages unused, whereas combining the two branches improves both restoration accuracy and spectral fidelity.

\subsubsection{Impact of Proxy Quality}
\begin{table}[h]
\centering
\setlength{\abovecaptionskip}{-0.0cm}
\setlength{\belowcaptionskip}{-0.2cm}
\small
\renewcommand\arraystretch{1.2}
\setlength{\tabcolsep}{2.6pt}
\captionsetup{font=scriptsize}
\caption{Impact of proxy quality. Base denotes the ICVL-trained restorer without target-adaptive fine-tuning. Weak and Strong denote HIR-ALIGN adaptation using weak and strong proxy generators, respectively.}
\label{tab:proxy_quality}
\resizebox{\linewidth}{!}{
\begin{tabular}{c|c|c|cc|cc}
\hline
 &  &  & \multicolumn{2}{c|}{CAVE} & \multicolumn{2}{c}{KAIST} \\
\hline
Degradation & Restorer & \makecell{Proxy Generator\\(Weak/Strong)}
& \makecell{Proxy PSNR\\(Weak/Strong)}
& \makecell{Restore PSNR\\(Base/Weak/Strong)}
& \makecell{Proxy PSNR\\(Weak/Strong)}
& \makecell{Restore PSNR\\(Base/Weak/Strong)} \\
\hline
Complex Noise
& SERT
& T3SC / SSUMamba
& 33.26/\textbf{39.61}
& 37.11/38.87/\textbf{40.29}
& 33.84/\textbf{39.17}
& 37.21/39.69/\textbf{41.45} \\
\hline
Super-resolution
& MCNet
& Bicubic / SFCSR
& 34.31/\textbf{37.11}
& 36.83/36.97/\textbf{37.65}
& 32.20/\textbf{34.57}
& 34.55/34.88/\textbf{35.13} \\
\hline
\end{tabular}
}
\vspace{-8.8pt}
\end{table}

We then study how proxy quality affects the final adaptation performance. The experiment is conducted on CAVE and KAIST under complex denoising and super-resolution. For complex denoising, SERT is used as the final restorer, while T3SC and SSUMamba are used as weak and strong proxy generators, respectively. For super-resolution, MCNet is used as the final restorer, while Bicubic and SFCSR are used as weak and strong proxy generators, respectively. Table~\ref{tab:proxy_quality} reports the proxy PSNR and the final restoration PSNR under the Base setting, which corresponds to the model pretrained only on ICVL without additional target-domain finetuning, as well as the Weak and Strong settings, which use weak and strong proxy generators for HIR-ALIGN adaptation, respectively.

The results show that stronger proxy generators produce higher-quality proxy HSIs and lead to better final restoration performance. This trend is expected because the synthesized HSIs inherit spectra from the proxy HSIs through the proposed spectral transfer module. Therefore, a more accurate proxy reduces the supervision error in both the proxy branch and the generated branch, making the adapted restorer more reliable.

Meanwhile, even the weak-proxy setting still improves over the Base model in all reported cases. This indicates that HIR-ALIGN does not require perfect proxy reconstructions. As long as the proxy preserves useful target-domain structures and approximate spectral information, the generated target-aligned samples can still provide effective supervision for adaptation. The stronger proxy further improves this process by providing cleaner spectral anchors for synthesis.

\subsubsection{Effects of Generated Data Volume}

\begin{table}[t]
\setlength{\tabcolsep}{4pt}
\centering
\caption{Ablation on the ratio $r$ of synthesized images to observed target images. Moderate synthesis improves coverage, whereas overly large ratios lead to saturation and slight degradation.}
\resizebox{0.7\linewidth}{!}{
\begin{tabular}{c|ccc|ccc}
\Xhline{1.2pt}
& \multicolumn{3}{c|}{Denoising} & \multicolumn{3}{c}{Super-resolution} \\
\hline
$r$ & PSNR & SSIM & SAM & PSNR & SSIM & SAM \\
\hline
1:1  & 40.15 & 0.950 & 10.10 & 35.15 & 0.952 & 3.31 \\
2:1  & 40.21 & 0.948 & 9.44  & 35.36 & 0.953 & 3.26 \\
3:1  & \textbf{40.59} & \textbf{0.958} & 8.90 & \textbf{35.63} & \textbf{0.955} & \textbf{3.16} \\
5:1  & 40.39 & 0.956 & 9.45 & 35.44 & 0.954 & 3.18 \\
10:1 & 40.46 & 0.948 & \textbf{8.71} & 35.30 & 0.953 & 3.35 \\
\Xhline{1.2pt}
\end{tabular}}
\label{tab:ab2}
\end{table}

We further investigate the effect of the amount of generated data. The experiment is conducted on KAIST using SERT for complex denoising and MCNet for super-resolution. We vary the ratio $r$ between synthesized images and observed target images from 1:1 to 10:1, while keeping the rest of the training protocol unchanged. The results are shown in Table~\ref{tab:ab2}.

As the synthesis ratio increases from 1:1 to 3:1, performance improves because more target-aligned views are introduced and the effective training distribution covers a larger portion of the target domain. This is exactly the regime in which synthesis is most useful: each additional generated sample contributes new appearance variation while still remaining anchored to the proxy spectral manifold.

When the ratio further increases to 5:1 and 10:1, the benefit saturates and slightly degrades. This saturation is theoretically meaningful. The coverage term cannot decrease indefinitely, because many generated samples are correlated views derived from the same small proxy set. Beyond a certain point, adding more generated data mainly increases redundancy and exposes the learner more strongly to generator mismatch and warp bias. The interior optimum observed in Table~\ref{tab:ab2} therefore matches the corollary derived from our unified risk bound.

\subsubsection{Effects of the Improved unCLIP}
\begin{table}[t]
\setlength{\tabcolsep}{5pt}
  \centering
  \caption{Ablation study on the effects of the improved unCLIP.}
  \resizebox{.85\columnwidth}{!}{
  \begin{tabular}{c|c c c |c c c}
    \Xhline{1.2pt}
    & \multicolumn{3}{c|}{Inpainting} & \multicolumn{3}{c}{Deblur} \\
    \hline
    Training set & PSNR & SSIM & SAM & PSNR & SSIM & SAM\\
    \hline
    Baseline (Src) & 36.36 &0.935 &7.72 & 40.90 &0.978 &2.98 \\
    UnCLIP (Original) & 36.92 & 0.922 & 7.68 & 41.10 & 0.978 & \textbf{2.93} \\
    UnCLIP (Ours) & \textbf{37.42} & \textbf{0.936} & \textbf{7.36}  & \textbf{41.67} & \textbf{0.980} & 2.96 \\
    \Xhline{1.2pt}
  \end{tabular}
}
\label{tab:ab4}
\end{table}
We next evaluate the effect of the improved unCLIP module used for target-aligned RGB generation. The experiment is conducted on KAIST with PromptIR under inpainting and deblurring. We compare three settings: ``Src'' baseline, HIR-ALIGN using the original unCLIP, and HIR-ALIGN using the improved unCLIP. The results are presented in Table~\ref{tab:ab4}. The ablation in Table~\ref{tab:ab4} shows that the proposed modifications to unCLIP are essential. Replacing the original unCLIP with the improved version yields clear gains on both inpainting and deblurring, confirming the value of degradation-robust CLIP conditioning, text-prompt guidance, and noise-perturbed initialization. These gains are consistent with the role of the RGB generator in our pipeline: if the RGB guides are poorly aligned with the target domain, the subsequent warp stage must compensate for mismatched appearance, which is both harder and less stable.

The improved unCLIP therefore does not need to produce a perfect final image by itself; rather, it needs to produce an RGB guide whose semantics and local appearance are sufficiently consistent with the target domain to support reliable spectral warping. The ablation verifies that the three proposed modifications improve exactly this property.

\subsubsection{Comparison with HSIGene-generated Training Data}
\begin{table}[t]
\renewcommand{\arraystretch}{1.1}
\setlength{\tabcolsep}{3pt}
\centering
\caption{Ablation on the generation source. We compare training with HSIGene-generated data and training with the proposed HIR-ALIGN data construction on CAVE and KAIST.}
\resizebox{\linewidth}{!}{
\begin{tabular}{c|c|ccc|ccc}
\Xhline{1.2pt}
 &  & \multicolumn{3}{c|}{CAVE} & \multicolumn{3}{c}{KAIST} \\
\hline
Degradation & Method & PSNR $\uparrow$ & SSIM $\uparrow$ & SAM $\downarrow$ & PSNR $\uparrow$ & SSIM $\uparrow$ & SAM $\downarrow$ \\
\hline
Complex Noise & SERT (HSIGene/Ours) & 37.15/\textbf{39.96} & 0.934/\textbf{0.959} & 13.93/\textbf{9.76} & 37.77/\textbf{40.59} & 0.935/\textbf{0.958} & 13.51/\textbf{8.90} \\
Super-resolution & MCNet (HSIGene/Ours) & 36.92/\textbf{37.65} & 0.952/\textbf{0.958} & 5.17/\textbf{3.76} & 34.47/\textbf{35.63} & 0.947/\textbf{0.955} & 4.14/\textbf{3.16} \\
\Xhline{1.2pt}
\end{tabular}}
\label{tab:hsigene_ablation}
\end{table}

We compare the proposed proxy-anchored generation strategy with a direct HSI generation baseline. The experiment is conducted on CAVE and KAIST using SERT for complex denoising and MCNet for super-resolution. In the HSIGene setting, the restorers are fine-tuned with HSIGene-generated data; in the HIR-ALIGN setting, they are fine-tuned with the target-aligned data constructed by our proxy-anchored synthesis pipeline. Table~\ref{tab:hsigene_ablation} reports the comparison. Directly using HSIGene-generated data improves the training distribution to some extent, but the proposed proxy-anchored synthesis yields higher PSNR and SSIM and lower SAM on both CAVE and KAIST. This result indicates that preserving target-domain spectra through semantic warping is more reliable for restoration-oriented adaptation than relying solely on a domain-specific HSI generator.

\subsubsection{Generalization on Unobserved Data}
\begin{table}[t]
\setlength{\tabcolsep}{3pt}
\centering
\caption{Ablation study on generalization capability under observed and unobserved sets.}
\resizebox{0.95\columnwidth}{!}{
\begin{tabular}{c | c | c c c | c c c}
\Xhline{1.2pt}
\multicolumn{2}{c|}{Tasks} & \multicolumn{3}{c|}{Denoising} & \multicolumn{3}{c}{Super-resolution} \\
\hline
Testsets & Training set & PSNR & SSIM & SAM & PSNR & SSIM & SAM \\
\hline
\multirow{2}{*}{Observed Set} 
& Src & 38.13 & 0.953 & 12.94 & 37.36 & 0.948 & 3.91 \\
& Ours & \textbf{38.45} & \textbf{0.953} & \textbf{11.04} & \textbf{37.50} & \textbf{0.949} & \textbf{3.89} \\
\hline
\multirow{2}{*}{Unobserved Set} 
& Src & 39.53 & 0.957 & 11.30 & 36.99 & 0.958 & 3.97 \\
& Ours & \textbf{40.13} & \textbf{0.961} & \textbf{9.38} & \textbf{37.03} & \textbf{0.959} & \textbf{3.94} \\
\hline
\multirow{2}{*}{Total Set} 
& Src & 39.08 & 0.955 & 11.83 & 37.11 & 0.955 & 3.95 \\
& Ours & \textbf{39.58} & \textbf{0.958} & \textbf{9.92} & \textbf{37.18} & \textbf{0.956} & \textbf{3.92} \\
\Xhline{1.2pt}
\end{tabular}
}
\label{tab:ab3}
\end{table}
We evaluate whether HIR-ALIGN generalizes beyond the observed target images used for data construction. The experiment is conducted on CAVE, where only 10 degraded images are treated as observed target samples for proxy generation and fine-tuning, while the remaining images are held out as unobserved test samples. We conduct complex denoising with SERT and super-resolution with MCNet.
Table~\ref{tab:ab3} shows consistent improvements on observed, unobserved, and total sets. The gain on the unobserved split is especially important: it indicates that HIR-ALIGN does not simply memorize the observed samples or overfit to the specific synthesized views generated from them. Instead, the adapted model learns a better approximation of the \emph{target distribution}. This finding supports the interpretation of HIR-ALIGN as a data-distribution alignment strategy rather than a sample memorization strategy. The synthesized views enrich the local neighborhood around the observed target samples, and the downstream restorer then leverages these views to improve generalization on new target images that were never used during synthesis.

\subsubsection{Clean-Target Upper-Bound Analysis}
\begin{table}[h]
\centering
\renewcommand\arraystretch{1}
\setlength{\tabcolsep}{2.5pt}
\caption{Clean-target reference comparison on the CAVE observed/unobserved split using proxy-only adaptation (Src), HIR-ALIGN (Ours), and clean-target finetuning (Clean).}
\label{tab:cave_kaist_ub}
\resizebox{\linewidth}{!}{
\begin{tabular}{c|c|ccc}
\hline
Degradation & Method & PSNR $\uparrow$  & SSIM $\uparrow$ & SAM $\downarrow$  \\
\hline
Complex Noise & SERT (Src/Ours/UB) & 39.53/40.13/\textbf{40.99} & 0.957/0.961/\textbf{0.969} & 11.30/\textbf{9.38}/9.73 \\
Super-resolution & MCNet (Src/Ours/UB) & 36.99/37.03/\textbf{37.52} & 0.958/0.959/\textbf{0.971} & 3.97/3.94/\textbf{3.23} \\
\hline
\end{tabular}
}
\vspace{-8.8pt}
\end{table}

Finally, we compare HIR-ALIGN with an oracle fine-tuning setting that uses clean target HSIs. This experiment estimates the remaining gap between our proxy-based adaptation and an ideal clean-target fine-tuning scenario. We conduct the comparison under complex denoising with SERT and super-resolution with MCNet, in which the setting is the same as the previous section, and report the results in Table~\ref{tab:cave_kaist_ub}. HIR-ALIGN remains below this oracle setting in most cases, as expected, but already closes a substantial fraction of the gap. The remaining difference to the upper bound can be interpreted as the combined effect of proxy noise, imperfect RGB generation, and the structural bias of the sparse warp. In other words, the oracle experiment does not invalidate the proxy-based strategy; instead, it quantifies how much performance is still available if those error sources were further reduced.

\section{Conclusion}
We presented HIR-ALIGN, a plug-and-play target-adaptive augmentation framework for hyperspectral image restoration. The framework first builds proxy HSIs from degraded target observations, then synthesizes target-aligned RGB images with improved unCLIP, and finally generates HSIs through a warp-based spectral transfer module that reuses proxy spectra via sparse matching, soft aggregation, and shared local interpolation. Extensive experiments on simulated and real-world datasets demonstrate the effectiveness of HIR-ALIGN across multiple restoration tasks. The resulting training set consisting of proxy and generated samples improves the generalization of existing restoration backbones under distribution shift.

\section{Acknowledgment}
The authors acknowledge the use of AI-assisted tools for assistance with code implementation, code debugging, and language polishing of the manuscript.

\bibliographystyle{IEEEtran}
\bibliography{references}

\vfill

\clearpage
\appendix
\section*{Supplementary Material}

\section{Sparse-Warp Formulation of HSI Generation}
\label{sec:supp_sparse_warp}

This section gives only the operator-level derivation needed by the later risk analysis. Let $p\in\mathbb{R}^{H\times W\times B}$ denote a proxy HSI and let $r=\Gamma(p)\in\mathbb{R}^{H\times W\times 3}$ be its proxy RGB. After flattening the spatial grid $\Omega$ with $n=HW$, write $p\in\mathbb{R}^{n\times B}$ and $r\in\mathbb{R}^{n\times 3}$. For a generated RGB guide $g$, the warp module retains a sparse candidate set $\mathcal C(u)=\{v_{u,1},\ldots,v_{u,K}\}$ for each guide pixel $u$, assigns nonnegative aggregation weights $a_{u,k}$, and uses nonnegative local interpolation weights $b_{u,\delta}$ on an $s\times s$ stencil $\Delta_s$:
\begin{equation}
 a_{u,k}\ge 0,
 \quad
 \sum_{k=1}^{K}a_{u,k}=1,
 \quad
 b_{u,\delta}\ge 0,
 \quad
 \sum_{\delta\in\Delta_s}b_{u,\delta}=1.
\label{eq:supp_compact_weights}
\end{equation}
The same learned weights are applied to the proxy RGB and to all spectral bands of the proxy HSI.

Define $A\in\mathbb{R}_{+}^{n\times n}$ by
\begin{equation}
A_{uv}=
\sum_{k=1}^{K} a_{u,k}\mathbf{1}\{v=v_{u,k}\}.
\label{eq:supp_compact_A}
\end{equation}
Then the pre-interpolation RGB and HSI are
\begin{equation}
\bar r=Ar,
\quad
\bar y=Ap.
\label{eq:supp_compact_preinterp}
\end{equation}
Define $B\in\mathbb{R}_{+}^{n\times n}$ by
\begin{equation}
B_{uv}=
\sum_{\delta\in\Delta_s} b_{u,\delta}\mathbf{1}\{v=\pi(u+\delta)\},
\label{eq:supp_compact_B}
\end{equation}
where $\pi(\cdot)$ maps boundary locations back to $\Omega$. The final transported RGB and HSI are
\begin{equation}
\tilde r=BAr,
\quad
\tilde y=BAp.
\label{eq:supp_compact_final}
\end{equation}
We therefore define the overall warp operator as
\begin{equation}
T:=BA.
\label{eq:supp_T_def}
\end{equation}

\begin{theorem}[Common sparse-warp representation]
\label{thm:common_sparse_warp}
There exists a single sparse operator $T=BA$ such that
\begin{equation}
\tilde r=Tr,
\quad
\tilde y=Tp.
\label{eq:supp_common_sparse_warp}
\end{equation}
Thus the generated RGB and generated HSI share the same spatial transport weights; they differ only in the proxy modality being transported.
\end{theorem}

\begin{proof}
Equation~\eqref{eq:supp_compact_preinterp} applies the same aggregation matrix $A$ to $r$ and $p$. Equation~\eqref{eq:supp_compact_final} then applies the same interpolation matrix $B$ to the two aggregated modalities. Composing the two operations gives $\tilde r=BAr$ and $\tilde y=BAp$, proving the claim with $T=BA$.
\end{proof}

\begin{proposition}[Nonnegativity, stochasticity, and sparsity]
\label{prop:sparse_stochastic}
The matrices $A$, $B$, and $T=BA$ are nonnegative and row-stochastic. Each row of $A$ has at most $K$ nonzero entries, each row of $B$ has at most $|\Delta_s|=s^2$ nonzero entries, and each row of $T$ has at most $K|\Delta_s|$ nonzero entries.
\end{proposition}

\begin{proof}
The nonnegativity of $A$ and $B$ follows from Eq.~\eqref{eq:supp_compact_weights}. Moreover,
\begin{equation}
\sum_{v=1}^{n}A_{uv}=
\sum_{k=1}^{K}a_{u,k}=1,
\quad
\sum_{v=1}^{n}B_{uv}=
\sum_{\delta\in\Delta_s}b_{u,\delta}=1.
\label{eq:supp_row_stochastic_compact}
\end{equation}
Hence $A$ and $B$ are row-stochastic, and their product $T=BA$ is also nonnegative and row-stochastic. The support bounds follow directly from the definitions of $A$ and $B$. If multiple candidate paths reach the same proxy pixel, the actual number of distinct nonzeros is smaller, never larger.
\end{proof}

\begin{corollary}[Convex-combination property]
\label{cor:convex_combination}
For every spatial position $u$ and spectral band $b$, the synthesized value $\tilde y(u,b)$ is a convex combination of at most $K|\Delta_s|$ proxy spectral values.
\end{corollary}

\begin{proof}
By Proposition~\ref{prop:sparse_stochastic}, $T$ is row-stochastic. Therefore,
\begin{equation}
\tilde y(u,b)=\sum_{v=1}^{n}T_{uv}p(v,b),
\quad
T_{uv}\ge 0,
\quad
\sum_{v=1}^{n}T_{uv}=1.
\label{eq:supp_convex_combo_compact}
\end{equation}
The support bound follows from the row sparsity of $T$.
\end{proof}

\begin{remark}[Spectral conservatism]
The generated spectra are produced by transporting proxy spectra through $T$. No separate RGB-to-HSI reconstruction branch is assumed in the theory. This is why the generated labels remain tied to the proxy spectral manifold.
\end{remark}

\section{Error Propagation Through the Sparse Warp}

Let the proxy error be
\begin{equation}
E:=p-y.
\label{eq:supp_proxy_error}
\end{equation}
For a fixed realized warp $T$, the clean counterpart of the generated label is $Ty$, whereas the actual generated label is $Tp$. Therefore,
\begin{equation}
E_g:=Tp-Ty=T(p-y)=TE.
\label{eq:supp_generated_error}
\end{equation}
This identity is central: the generated label error is exactly the proxy error after the same sparse warp.

\begin{assumption}[Bounded overlap]
\label{ass:supp_bounded_overlap}
The sparse warp satisfies
\begin{equation}
T^\top T\preceq \kappa I_n
\label{eq:supp_bounded_overlap}
\end{equation}
for some $\kappa\ge 1$. Equivalently,
\begin{equation}
\|TZ\|_F^2
\le
\kappa\|Z\|_F^2
\quad
\text{for all }Z\in\mathbb{R}^{n\times B}.
\label{eq:supp_overlap_frob}
\end{equation}
\end{assumption}

The parameter $\kappa$ measures how strongly output pixels reuse the same proxy spectra. It is near one when the warp behaves like a permutation and can be larger when many output pixels collapse onto similar proxy locations.

\begin{lemma}[Noise inheritance under sparse warping]
\label{lem:noise_inheritance}
Define
\begin{equation}
N_p:=2\,\mathbb{E}\|E\|_F^2,
\quad
N_g:=2\,\mathbb{E}\|E_g\|_F^2.
\label{eq:supp_noise_terms}
\end{equation}
Under Assumption~\ref{ass:supp_bounded_overlap},
\begin{equation}
N_g\le \kappa N_p.
\label{eq:supp_noise_control}
\end{equation}
\end{lemma}

\begin{proof}
Using Eq.~\eqref{eq:supp_generated_error} and Assumption~\ref{ass:supp_bounded_overlap},
\begin{equation}
\|E_g\|_F^2
=
\|TE\|_F^2
\le
\kappa\|E\|_F^2.
\label{eq:supp_noise_inherit_step}
\end{equation}
Taking expectations and multiplying both sides by $2$ yields Eq.~\eqref{eq:supp_noise_control}.
\end{proof}

\begin{remark}[Coordinate-wise interpretation]
Because $T$ is row-stochastic,
\begin{equation}
E_g(u,b)
=
\sum_{v=1}^{n}T_{uv}E(v,b).
\label{eq:supp_coordinate_error}
\end{equation}
Thus each generated error is an average of proxy errors over the support of row $u$ of $T$. In particular,
\begin{equation}
|E_g(u,b)|
\le
\max_{v\in\mathrm{supp}(T_{u,:})}
|E(v,b)|.
\label{eq:supp_coordinate_bound}
\end{equation}
The warp redistributes proxy errors; it does not create unconstrained spectral values.
\end{remark}

\section{Distribution-Level Setup}

We now lift the operator-level results to the distribution level used by aligned finetuning. Throughout this section, $\Wass(P,Q)$ denotes the 1-Wasserstein distance under the pair metric
\begin{equation}
d\bigl((x,y),(x',y')\bigr)
=
\|x-x'\|_F+\|y-y'\|_F.
\label{eq:supp_pair_metric}
\end{equation}

\subsection{Clean and Actual Training Distributions}

Let $Y\sim\mu_t$ denote a clean HSI from the full target distribution, and let $Y_o\sim\mu_o$ denote a clean HSI from the observed target subset used for adaptation. The full target test distribution is
\begin{equation}
P_t:=\mathcal{L}\bigl(\Dt(Y),Y\bigr).
\label{eq:supp_target_dist}
\end{equation}
For an observed scene, the proxy is written as $P_o=Y_o+E$. The clean counterpart and the actual proxy branch are
\begin{equation}
\begin{aligned}
P_p^{\star}
&:=\mathcal{L}\bigl(\Dt(Y_o),Y_o\bigr),\\
P_p
&:=\mathcal{L}\bigl(\Dt(P_o),P_o\bigr).
\end{aligned}
\label{eq:supp_proxy_branch_dist}
\end{equation}
For the generated branch, let $T$ be the realized sparse warp estimated for the generated guide and then held fixed for the clean-counterpart comparison. We define
\begin{equation}
\begin{aligned}
P_g^{\star}
&:=\mathcal{L}\bigl(\Dt(TY_o),TY_o\bigr),\\
P_g
&:=\mathcal{L}\bigl(\Dt(TP_o),TP_o\bigr).
\end{aligned}
\label{eq:supp_generated_branch_dist}
\end{equation}
Freezing the same $T$ in both distributions makes the identity $TP_o-TY_o=TE$ exact.

Let $\rho$ denote an ideal family of admissible target-aligned semantic warps. The corresponding vicinal distribution is
\begin{equation}
P_v
:=
\mathcal{L}\bigl(\Dt(\widetilde{T}Y_o),
\widetilde{T}Y_o\bigr),
\quad
\widetilde{T}\sim\rho.
\label{eq:supp_vicinal_dist}
\end{equation}
The coverage and generator-mismatch quantities are
\begin{equation}
\begin{aligned}
\Delta_p&:=\Wass(P_t,P_p^{\star}),\\
\Delta_w&:=\Wass(P_t,P_v),\\
e_w&:=\Wass(P_g^{\star},P_v).
\end{aligned}
\label{eq:supp_wasserstein_terms}
\end{equation}
Here $\Delta_p$ measures the support gap of the observed target subset, $\Delta_w$ measures the gap of the ideal semantic-warp distribution, and $e_w$ measures how far the realized sparse warp is from the ideal vicinal family.

\subsection{Mixture Training Distribution}

Let $\alpha\in[0,1]$ be the fraction of generated samples in the finetuning set. The clean and actual mixed training distributions are
\begin{equation}
P_{\alpha}^{\star}
:=
(1-\alpha)P_p^{\star}
+\alpha P_g^{\star},
\label{eq:supp_mixture_clean}
\end{equation}
and
\begin{equation}
P_{\alpha}
:=
(1-\alpha)P_p+\alpha P_g.
\label{eq:supp_mixture_actual}
\end{equation}

\begin{lemma}[Coverage of the clean mixed distribution]
\label{lem:mixture_gap}
The clean mixed distribution satisfies
\begin{equation}
\Wass(P_t,P_{\alpha}^{\star})
\le
(1-\alpha)\Delta_p
+\alpha\Delta_w+\alpha e_w.
\label{eq:supp_mixture_gap}
\end{equation}
\end{lemma}

\begin{proof}
By convexity of the Wasserstein distance in its second argument,
\begin{equation}
\begin{aligned}
\Wass(P_t,P_{\alpha}^{\star})
&\le
(1-\alpha)\Wass(P_t,P_p^{\star})\\
&\quad+
\alpha\Wass(P_t,P_g^{\star}).
\end{aligned}
\label{eq:supp_mixture_convexity}
\end{equation}
By the triangle inequality,
\begin{equation}
\Wass(P_t,P_g^{\star})
\le
\Wass(P_t,P_v)+\Wass(P_v,P_g^{\star})
=
\Delta_w+e_w.
\label{eq:supp_generated_triangle}
\end{equation}
Substitution gives Eq.~\eqref{eq:supp_mixture_gap}.
\end{proof}

\begin{assumption}[Lipschitz degradation]
\label{ass:supp_lipschitz_degradation}
The task degradation operator is $L_D$-Lipschitz:
\begin{equation}
\|\Dt(z_1)-\Dt(z_2)\|_F
\le
L_D\|z_1-z_2\|_F
\quad
\text{for all }z_1,z_2.
\label{eq:supp_lipschitz_degradation}
\end{equation}
\end{assumption}
For a distribution $Q$ over pairs, define
\begin{equation}
R_Q(f)
:=
\mathbb{E}_{(x,y)\sim Q}
\bigl[\ell(f(x),y)\bigr],
\quad
R_t(f):=R_{P_t}(f).
\label{eq:supp_risk_def}
\end{equation}

\begin{assumption}[Lipschitz loss]
\label{ass:supp_lipschitz_loss}
For every hypothesis $f$ in the training class, the loss-induced risk is $L$-Lipschitz with respect to $\Wass$:
\begin{equation}
|R_P(f)-R_Q(f)|
\le
L\,\Wass(P,Q)
\label{eq:supp_lipschitz_shift}
\end{equation}
for any two pair distributions $P$ and $Q$.
\end{assumption}

\begin{lemma}[Sufficient condition for Assumption~\ref{ass:supp_lipschitz_loss}]
\label{lem:supp_lipschitz_from_integrand}
Suppose that for every fixed $f\in\mathcal F$, the map
\begin{equation}
\phi_f(x,y):=\ell(f(x),y)
\label{eq:supp_integrand_def}
\end{equation}
is $L$-Lipschitz under the pair metric \eqref{eq:supp_pair_metric}, i.e.,
\begin{equation}
|\phi_f(x,y)-\phi_f(x',y')|
\le
L\,d\bigl((x,y),(x',y')\bigr).
\label{eq:supp_integrand_lipschitz}
\end{equation}
Then Assumption~\ref{ass:supp_lipschitz_loss} holds.
\end{lemma}

\begin{proof}
Let $\pi$ be any coupling of $P$ and $Q$. Then
\begin{equation}
\begin{aligned}
|R_P(f)-R_Q(f)|
&=
\left|
\mathbb{E}_{((x,y),(x',y'))\sim \pi}
\bigl[\phi_f(x,y)-\phi_f(x',y')\bigr]
\right|\\
&\le
\mathbb{E}_{\pi}
\left|
\phi_f(x,y)-\phi_f(x',y')
\right|\\
&\le
L\,
\mathbb{E}_{\pi}
d\bigl((x,y),(x',y')\bigr).
\end{aligned}
\label{eq:supp_lipschitz_coupling_step}
\end{equation}
Taking the infimum over all couplings $\pi$ gives
\begin{equation}
|R_P(f)-R_Q(f)|
\le
L\,\Wass(P,Q).
\label{eq:supp_lipschitz_coupling_final}
\end{equation}
This is exactly Assumption~\ref{ass:supp_lipschitz_loss}.
\end{proof}

\begin{lemma}[Actual-vs-clean pair perturbation]
\label{lem:pair_perturb}
Under Assumptions~\ref{ass:supp_lipschitz_degradation} and~\ref{ass:supp_bounded_overlap},
\begin{equation}
\Wass(P_p,P_p^{\star})
\le
(1+L_D)\,\mathbb{E}\|E\|_F,
\label{eq:supp_proxy_pair_gap}
\end{equation}
and
\begin{equation}
\begin{aligned}
\Wass(P_g,P_g^{\star})
&\le
(1+L_D)\,\mathbb{E}\|TE\|_F\\
&\le
(1+L_D)
\sqrt{\kappa\,\mathbb{E}\|E\|_F^2}.
\end{aligned}
\label{eq:supp_generated_pair_gap}
\end{equation}
\end{lemma}

\begin{proof}
For the proxy branch, couple $(\Dt(P_o),P_o)$ with $(\Dt(Y_o),Y_o)$ using the same observed target scene. Since $P_o-Y_o=E$,
\begin{equation}
\begin{aligned}
&d\bigl((\Dt(P_o),P_o),
        (\Dt(Y_o),Y_o)\bigr)\\
&\quad =
\|\Dt(P_o)-\Dt(Y_o)\|_F+\|E\|_F\\
&\quad \le
(1+L_D)\|E\|_F .
\end{aligned}
\label{eq:supp_proxy_coupling}
\end{equation}
Taking expectations over this coupling yields Eq.~\eqref{eq:supp_proxy_pair_gap}.

For the generated branch, couple $(\Dt(TP_o),TP_o)$ with $(\Dt(TY_o),TY_o)$ under the same realized $T$. Since $TP_o-TY_o=TE$,
\begin{equation}
\begin{aligned}
&d\bigl((\Dt(TP_o),TP_o),
        (\Dt(TY_o),TY_o)\bigr)\\
&\quad =
\|\Dt(TP_o)-\Dt(TY_o)\|_F+\|TE\|_F\\
&\quad \le
(1+L_D)\|TE\|_F .
\end{aligned}
\label{eq:supp_generated_coupling}
\end{equation}
Taking expectations gives the first inequality in Eq.~\eqref{eq:supp_generated_pair_gap}. The second follows from Jensen's inequality and Assumption~\ref{ass:supp_bounded_overlap}:
\begin{equation}
\mathbb{E}\|TE\|_F
\le
\sqrt{\mathbb{E}\|TE\|_F^2}
\le
\sqrt{\kappa\,\mathbb{E}\|E\|_F^2}.
\label{eq:supp_generated_jensen}
\end{equation}
\end{proof}

\section{Unified Risk Bound}

The following theorem is the detailed version of the theoretical statement in the main paper. It is written for the actual supervised pairs used by HIR-ALIGN, namely $(\Dt(P_o),P_o)$ and $(\Dt(TP_o),TP_o)$.

\subsection{Comparator and finite-sample terms}

This subsection expands the two ingredients used in the unified risk bound. The first ingredient is a deterministic comparison on the clean mixture. It is not a property of the sparse-warp matrix alone; it follows from a fixed comparator and from how the two risk-level bias terms are defined. The second ingredient is the standard oracle-type excess-risk statement for approximate empirical risk minimization. 

\begin{lemma}[Clean-mixture comparator and warp bias]
\label{lem:supp_structural_bias}
Let $F^{\star}$ be the Bayes restorer under the full target distribution $P_t$. Let $f_{\mathcal F}^{\star}\in\mathcal F$ be a fixed comparator in the chosen hypothesis class. Define
\begin{equation}
\begin{aligned}
B_0^2
&:=
\left[
R_{P_p^{\star}}(f_{\mathcal F}^{\star})
-
R_t(F^{\star})
\right]_+,\\
B_{\mathrm{warp}}^2
&:=
\left[
R_{P_g^{\star}}(f_{\mathcal F}^{\star})
-
R_{P_p^{\star}}(f_{\mathcal F}^{\star})
\right]_+,
\end{aligned}
\label{eq:supp_bias_definitions}
\end{equation}
where $[a]_+:=\max\{a,0\}$. Then, for any $\alpha\in[0,1]$,
\begin{equation}
R_{P_{\alpha}^{\star}}(f_{\mathcal F}^{\star})
\le
R_t(F^{\star})
+
B_0^2
+
\alpha B_{\mathrm{warp}}^2 .
\label{eq:supp_structural_bias}
\end{equation}
\end{lemma}

\begin{proof}
By the definition of the clean mixture distribution,
\begin{equation}
P_{\alpha}^{\star}
=
(1-\alpha)P_p^{\star}
+
\alpha P_g^{\star}.
\label{eq:supp_clean_mixture_repeat}
\end{equation}
Since risk is an expectation and expectation is linear with respect to mixtures of distributions, for any fixed predictor $f$,
\begin{equation}
R_{P_{\alpha}^{\star}}(f)
=
(1-\alpha)R_{P_p^{\star}}(f)
+
\alpha R_{P_g^{\star}}(f).
\label{eq:supp_risk_mixture_linearity}
\end{equation}
Taking $f=f_{\mathcal F}^{\star}$ gives
\begin{equation}
\begin{aligned}
R_{P_{\alpha}^{\star}}(f_{\mathcal F}^{\star})
&=
(1-\alpha)R_{P_p^{\star}}(f_{\mathcal F}^{\star})
+
\alpha R_{P_g^{\star}}(f_{\mathcal F}^{\star})\\
&=
R_{P_p^{\star}}(f_{\mathcal F}^{\star})
+
\alpha
\Bigl[
R_{P_g^{\star}}(f_{\mathcal F}^{\star})
-
R_{P_p^{\star}}(f_{\mathcal F}^{\star})
\Bigr].
\end{aligned}
\label{eq:supp_structural_bias_expand}
\end{equation}
From Eq.~\eqref{eq:supp_bias_definitions},
\begin{equation}
R_{P_p^{\star}}(f_{\mathcal F}^{\star})
\le
R_t(F^{\star})+B_0^2,
\label{eq:supp_b0_upper}
\end{equation}
and
\begin{equation}
R_{P_g^{\star}}(f_{\mathcal F}^{\star})
-
R_{P_p^{\star}}(f_{\mathcal F}^{\star})
\le
B_{\mathrm{warp}}^2.
\label{eq:supp_bwarp_upper}
\end{equation}
Substituting Eqs.~\eqref{eq:supp_b0_upper} and \eqref{eq:supp_bwarp_upper} into Eq.~\eqref{eq:supp_structural_bias_expand} gives Eq.~\eqref{eq:supp_structural_bias}.
\end{proof}

\begin{remark}[Meaning of the two bias terms]
The term $B_0^2$ measures the positive excess risk of the comparator on the clean proxy branch relative to the Bayes target risk. It includes both the mismatch between the observed target subset and the full target distribution, and the possible approximation gap caused by restricting the predictor to $\mathcal F$. The term $B_{\mathrm{warp}}^2$ measures the additional positive excess risk incurred when the clean proxy branch is replaced by the clean generated branch. The factor $\alpha$ appears because only an $\alpha$ fraction of the clean mixture comes from the generated branch. This is why $B_{\mathrm{warp}}^2$ should be interpreted as a risk-level structural allowance, not as a norm or algebraic property of the matrix $T$.
\end{remark}

Let $S_{\alpha}=\{(x_i,y_i)\}_{i=1}^{N_{\alpha}}$ denote the finite mixed finetuning sample whose population distribution is $P_{\alpha}$. The empirical risk on this sample is
\begin{equation}
\widehat R_{S_{\alpha}}(f)
:=
\frac{1}{N_{\alpha}}
\sum_{i=1}^{N_{\alpha}}
\ell(f(x_i),y_i).
\label{eq:supp_empirical_risk}
\end{equation}

\begin{assumption}[Approximate empirical risk minimization]
\label{ass:supp_approx_erm}
The training algorithm returns $f_{\alpha}\in\mathcal F$ satisfying
\begin{equation}
\widehat R_{S_{\alpha}}(f_{\alpha})
\le
\inf_{f\in\mathcal F}
\widehat R_{S_{\alpha}}(f)
+
\eta_{\alpha},
\label{eq:supp_approx_erm}
\end{equation}
where $\eta_{\alpha}\ge0$ is the empirical optimization error.
\end{assumption}

\begin{assumption}[Uniform convergence on the mixed sample]
\label{ass:supp_uniform_convergence}
For the mixed sample $S_{\alpha}$, the empirical risk uniformly approximates the population risk over $\mathcal F$:
\begin{equation}
\sup_{f\in\mathcal F}
\left|
R_{P_{\alpha}}(f)
-
\widehat R_{S_{\alpha}}(f)
\right|
\le
\Gamma_{\alpha}.
\label{eq:supp_uniform_conv}
\end{equation}
Here, $\Gamma_{\alpha}$ denotes the uniform-convergence/generalization error of the mixed sample $S_{\alpha}$. This is the standard uniform-convergence condition used in empirical risk minimization and oracle-type excess-risk bounds~\cite{vapnik1971uniform,shalev2014understanding,mohri2018foundations,koltchinskii2011oracle}.
\end{assumption}

\begin{lemma}[Actual-mixture optimization/generalization error]
\label{lem:supp_opt_gap}
Under Assumptions~\ref{ass:supp_approx_erm} and \ref{ass:supp_uniform_convergence}, the learned predictor satisfies
\begin{equation}
R_{P_{\alpha}}(f_{\alpha})
\le
\inf_{f\in\mathcal F}R_{P_{\alpha}}(f)
+
\eta_{\alpha}
+
2\Gamma_{\alpha}.
\label{eq:supp_oracle_ineq_general}
\end{equation}
By standard ERM uniform-convergence and oracle-inequality theory, the optimization and generalization errors admit a high-probability finite-sample complexity bound \cite{vapnik1971uniform,shalev2014understanding,mohri2018foundations,koltchinskii2011oracle} with probability at least $1-\delta$,
\begin{equation}
\eta_\alpha + 2\Gamma_\alpha
\le
\psi_\delta\big(m_{\mathrm{eff}}(\alpha),\mathcal F,\ell\big).
\end{equation}
For simplicity, we use an effective-rate specialization of this generic bound and assume that there exists a constant $C_v\ge 0$ such that
\begin{equation}
\eta_{\alpha}+2\Gamma_{\alpha}
\le
\frac{C_v}{m_{\mathrm{eff}}(\alpha)},
\label{eq:supp_effective_rate}
\end{equation}
where $C_v$ absorbs the dependence on the hypothesis class, loss, confidence level, and sample-dependence structure, while $m_{\mathrm{eff}}(\alpha)$ denotes the effective sample size. Then
\begin{equation}
R_{P_{\alpha}}(f_{\alpha})
\le
\inf_{f\in\mathcal F}R_{P_{\alpha}}(f)
+
\frac{C_v}{m_{\mathrm{eff}}(\alpha)}.
\label{eq:supp_bias_variance}
\end{equation}
\end{lemma}

\begin{proof}
By Assumption~\ref{ass:supp_uniform_convergence}, applied to $f_{\alpha}$,
\begin{equation}
R_{P_{\alpha}}(f_{\alpha})
\le
\widehat R_{S_{\alpha}}(f_{\alpha})
+
\Gamma_{\alpha}.
\label{eq:supp_oracle_step1}
\end{equation}
By Assumption~\ref{ass:supp_approx_erm},
\begin{equation}
\widehat R_{S_{\alpha}}(f_{\alpha})
\le
\inf_{f\in\mathcal F}
\widehat R_{S_{\alpha}}(f)
+
\eta_{\alpha}.
\label{eq:supp_oracle_step2}
\end{equation}
Let
\begin{equation}
f_{\alpha}^{\star}
\in
\arg\min_{f\in\mathcal F}
R_{P_{\alpha}}(f)
\label{eq:supp_pop_minimizer}
\end{equation}
be a population-risk minimizer in $\mathcal F$. If the minimizer does not exist, the same argument applies to an arbitrarily close $\epsilon$-minimizer and then lets $\epsilon\downarrow0$. Since an infimum is no larger than the value at any particular function,
\begin{equation}
\inf_{f\in\mathcal F}
\widehat R_{S_{\alpha}}(f)
\le
\widehat R_{S_{\alpha}}(f_{\alpha}^{\star}).
\label{eq:supp_oracle_step3}
\end{equation}
Again by Assumption~\ref{ass:supp_uniform_convergence}, now applied to $f_{\alpha}^{\star}$,
\begin{equation}
\widehat R_{S_{\alpha}}(f_{\alpha}^{\star})
\le
R_{P_{\alpha}}(f_{\alpha}^{\star})
+
\Gamma_{\alpha}.
\label{eq:supp_oracle_step4}
\end{equation}
Combining Eqs.~\eqref{eq:supp_oracle_step1}--\eqref{eq:supp_oracle_step4} gives
\begin{equation}
\begin{aligned}
R_{P_{\alpha}}(f_{\alpha})
&\le
\widehat R_{S_{\alpha}}(f_{\alpha})
+
\Gamma_{\alpha}\\
&\le
\inf_{f\in\mathcal F}
\widehat R_{S_{\alpha}}(f)
+
\eta_{\alpha}
+
\Gamma_{\alpha}\\
&\le
\widehat R_{S_{\alpha}}(f_{\alpha}^{\star})
+
\eta_{\alpha}
+
\Gamma_{\alpha}\\
&\le
R_{P_{\alpha}}(f_{\alpha}^{\star})
+
\eta_{\alpha}
+
2\Gamma_{\alpha}\\
&=
\inf_{f\in\mathcal F}
R_{P_{\alpha}}(f)
+
\eta_{\alpha}
+
2\Gamma_{\alpha}.
\end{aligned}
\label{eq:supp_oracle_chain}
\end{equation}
This proves Eq.~\eqref{eq:supp_oracle_ineq_general}. Substituting Eq.~\eqref{eq:supp_effective_rate} gives Eq.~\eqref{eq:supp_bias_variance}.
\end{proof}

\begin{remark}[Effective sample size]
The quantity $m_{\mathrm{eff}}(\alpha)$ is not necessarily the raw number of proxy and generated samples. It represents the number of statistically useful views after accounting for correlation and redundancy among generated examples. Highly correlated generated views can make $m_{\mathrm{eff}}(\alpha)$ much smaller than the raw sample count, while genuinely diverse target-aligned views increase it. The specific decay rate of the statistical error depends on the complexity of $\mathcal F$, the loss, and the dependence structure of the generated samples. The form $C_v/m_{\mathrm{eff}}(\alpha)$ is therefore used as an oracle-type effective-view term.
\end{remark}

Define
\begin{equation}
\begin{aligned}
\Delta_{\alpha}
&:=
(1-\alpha)\Delta_p+\alpha\Delta_w+\alpha e_w,\\
\varepsilon(\alpha)
&:=
(1-\alpha)\varepsilon_p+\alpha\varepsilon_g,
\end{aligned}
\label{eq:supp_delta_epsilon_alpha}
\end{equation}
where
\begin{equation}
\varepsilon_p:=\Wass(P_p,P_p^{\star}),
\quad
\varepsilon_g:=\Wass(P_g,P_g^{\star}).
\label{eq:supp_noise_mix}
\end{equation}

\begin{theorem}[Unified target-risk bound]
\label{thm:supp_main}
Under Assumptions~\ref{ass:supp_lipschitz_loss}, \ref{ass:supp_approx_erm}, and~\ref{ass:supp_uniform_convergence}, and under the effective-rate condition in Eq.~\eqref{eq:supp_effective_rate},
\begin{equation}
\begin{aligned}
R_t(f_{\alpha})
\le\;&
R_t(F^{\star})
+B_0^2
+\alpha B_{\mathrm{warp}}^2\\
&+
\frac{C_v}{m_{\mathrm{eff}}(\alpha)}
+L\Delta_{\alpha}
+2L\varepsilon(\alpha).
\end{aligned}
\label{eq:supp_main_bound}
\end{equation}
Moreover, under Assumptions~\ref{ass:supp_lipschitz_degradation} and~\ref{ass:supp_bounded_overlap},
\begin{equation}
\begin{aligned}
\varepsilon_p
&\le
(1+L_D)\,\mathbb{E}\|E\|_F,\\
\varepsilon_g
&\le
(1+L_D)
\sqrt{\kappa\,\mathbb{E}\|E\|_F^2}.
\end{aligned}
\label{eq:supp_noise_mix_control}
\end{equation}
\end{theorem}

\begin{proof}
By Assumption~\ref{ass:supp_lipschitz_loss},
\begin{equation}
R_t(f_{\alpha})
\le
R_{P_{\alpha}}(f_{\alpha})
+
L\,\Wass(P_t,P_{\alpha}).
\label{eq:supp_target_to_actual}
\end{equation}
Using Lemma~\ref{lem:supp_opt_gap} and evaluating the infimum at $f_{\mathcal F}^{\star}$ gives
\begin{equation}
R_{P_{\alpha}}(f_{\alpha})
\le
R_{P_{\alpha}}(f_{\mathcal F}^{\star})
+
\frac{C_v}{m_{\mathrm{eff}}(\alpha)}.
\label{eq:supp_opt_step}
\end{equation}
Again by Assumption~\ref{ass:supp_lipschitz_loss},
\begin{equation}
R_{P_{\alpha}}(f_{\mathcal F}^{\star})
\le
R_{P_{\alpha}^{\star}}(f_{\mathcal F}^{\star})
+
L\,\Wass(P_{\alpha},P_{\alpha}^{\star}).
\label{eq:supp_clean_compare}
\end{equation}
Combining Eqs.~\eqref{eq:supp_target_to_actual}--\eqref{eq:supp_clean_compare} yields
\begin{equation}
\begin{aligned}
R_t(f_{\alpha})
\le\;&
R_{P_{\alpha}^{\star}}(f_{\mathcal F}^{\star})
+
\frac{C_v}{m_{\mathrm{eff}}(\alpha)}\\
&+
L\,\Wass(P_t,P_{\alpha})
+
L\,\Wass(P_{\alpha},P_{\alpha}^{\star}).
\end{aligned}
\label{eq:supp_after_opt}
\end{equation}
The triangle inequality gives
\begin{equation}
\Wass(P_t,P_{\alpha})
\le
\Wass(P_t,P_{\alpha}^{\star})
+
\Wass(P_{\alpha}^{\star},P_{\alpha}).
\label{eq:supp_triangle_actual}
\end{equation}
Substituting Eq.~\eqref{eq:supp_triangle_actual} into Eq.~\eqref{eq:supp_after_opt} gives
\begin{equation}
\begin{aligned}
R_t(f_{\alpha})
\le\;&
R_{P_{\alpha}^{\star}}(f_{\mathcal F}^{\star})
+
\frac{C_v}{m_{\mathrm{eff}}(\alpha)}\\
&+
L\,\Wass(P_t,P_{\alpha}^{\star})
+
2L\,\Wass(P_{\alpha}^{\star},P_{\alpha}).
\end{aligned}
\label{eq:supp_after_transport}
\end{equation}
Lemma~\ref{lem:supp_structural_bias} bounds the first term on the right. Lemma~\ref{lem:mixture_gap} gives
\begin{equation}
\Wass(P_t,P_{\alpha}^{\star})
\le
\Delta_{\alpha}.
\label{eq:supp_delta_alpha_bound}
\end{equation}
By convexity of $\Wass$ in each argument,
\begin{equation}
\begin{aligned}
\Wass(P_{\alpha}^{\star},P_{\alpha})
&\le
(1-\alpha)\Wass(P_p^{\star},P_p)\\
&\quad+
\alpha\Wass(P_g^{\star},P_g)\\
&=
\varepsilon(\alpha).
\end{aligned}
\label{eq:supp_eps_alpha_bound}
\end{equation}
Substituting Eqs.~\eqref{eq:supp_delta_alpha_bound} and \eqref{eq:supp_eps_alpha_bound} into Eq.~\eqref{eq:supp_after_transport} proves Eq.~\eqref{eq:supp_main_bound}. Finally, Eq.~\eqref{eq:supp_noise_mix_control} follows directly from Lemma~\ref{lem:pair_perturb}.
\end{proof}

\section{Implementation Details of Improved unCLIP}
\label{sec:supp_unclip_details}

This section provides additional implementation details of the improved unCLIP module used for target-aligned RGB synthesis. We fine-tune the CLIP-vision encoder of unCLIP on COCO2017~\cite{COCO2017} using synthetically blurred image pairs. Specifically, images from the COCO2017 training set are randomly split into training and validation subsets with a ratio of 8:2. For each clean image, Gaussian blur is applied to construct the degraded counterpart. The blur kernel size is randomly selected from $\{3,5,7\}$, and the blur standard deviation is uniformly sampled from $[0.5,1.5]$.

During fine-tuning, the trainable CLIP-vision encoder takes the blurred image as input, while the original frozen CLIP-vision encoder extracts the target embedding from the corresponding clean image. The encoder is optimized with an embedding-alignment loss:
\begin{equation}
\mathcal{L}_{\mathrm{clip}}
=
\left\|
E_{\mathrm{clip}}^{\mathrm{ft}}(I_{\mathrm{blur}})
-
E_{\mathrm{clip}}^{\mathrm{frozen}}(I_{\mathrm{clean}})
\right\|_1 ,
\end{equation}
where $E_{\mathrm{clip}}^{\mathrm{ft}}$ denotes the fine-tuned CLIP-vision encoder and $E_{\mathrm{clip}}^{\mathrm{frozen}}$ denotes the original frozen encoder. This objective encourages blurred inputs to produce embeddings close to those of their clean counterparts, improving the robustness of image-conditioned generation under imperfect proxy RGBs.

We train the encoder for 10 epochs using Adam with a batch size of 8. The initial learning rate is set to $1\times10^{-6}$ and decayed by a cosine annealing schedule to a minimum learning rate of $1\times10^{-8}$. We use this conservative learning rate to adapt the encoder to blurred inputs while avoiding excessive drift from the original unCLIP embedding space.

During RGB generation, we use ``high resolution, highly detailed, 4k'' as the fixed positive prompt and ``blurry, horror, noise'' as the fixed negative prompt. Reverse diffusion is initialized from image embeddings perturbed with Gaussian noise with a standard deviation of 0.8 rather than from pure noise. For each proxy sample, we generate three target-aligned RGB guides for subsequent warp-based hyperspectral synthesis.

\section{Additional Generation and Warped Results}
\label{sec:supp_generation_warped}
\begin{figure*}[!t]
    \centering
    \includegraphics[width=\textwidth]{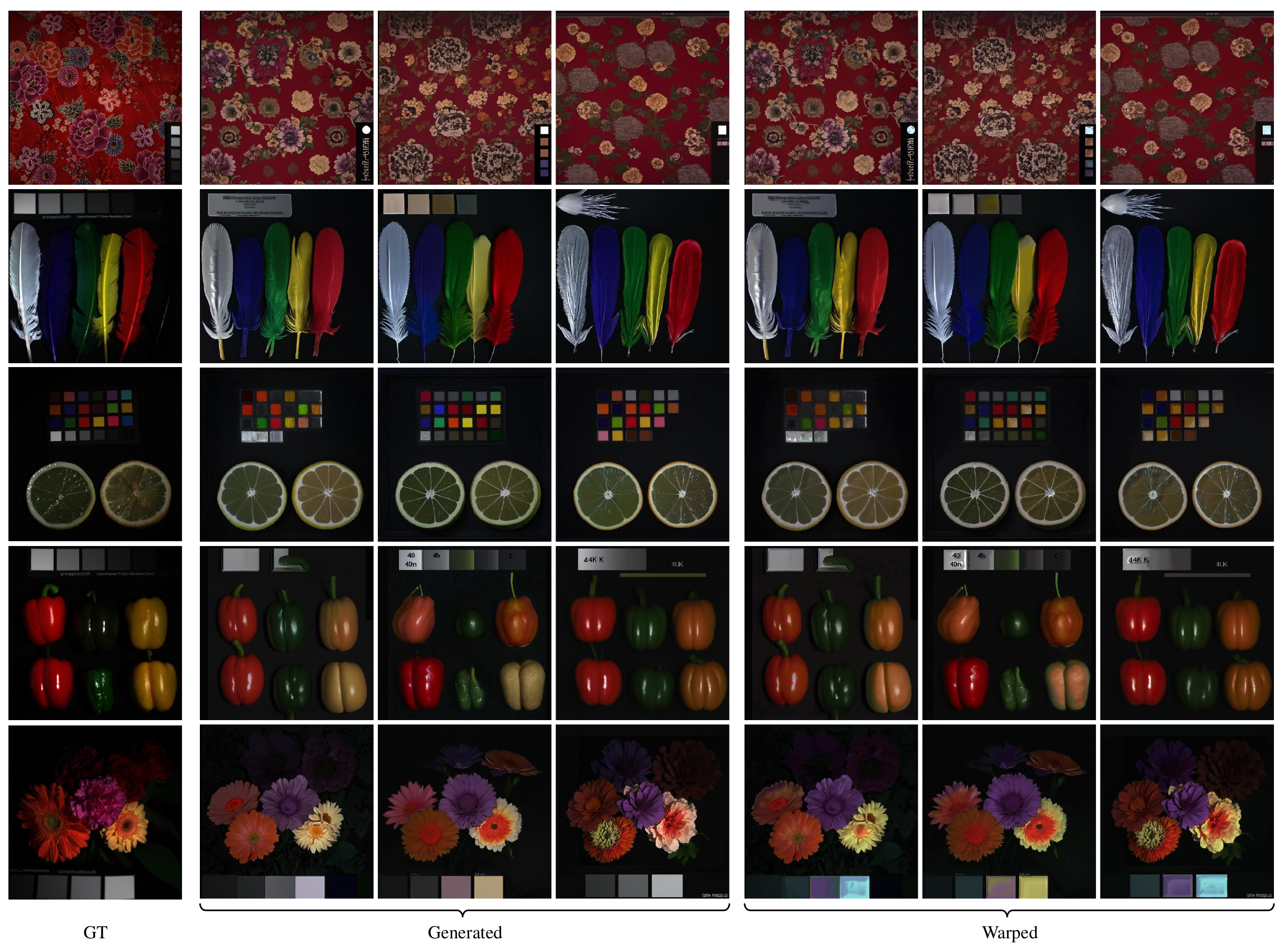}
    \caption{
    Additional generation and warped results on CAVE.
    }
    \label{fig:supp_cave_generation_warped}
\end{figure*}
\begin{figure*}[!t]
    \centering
    \includegraphics[width=\textwidth]{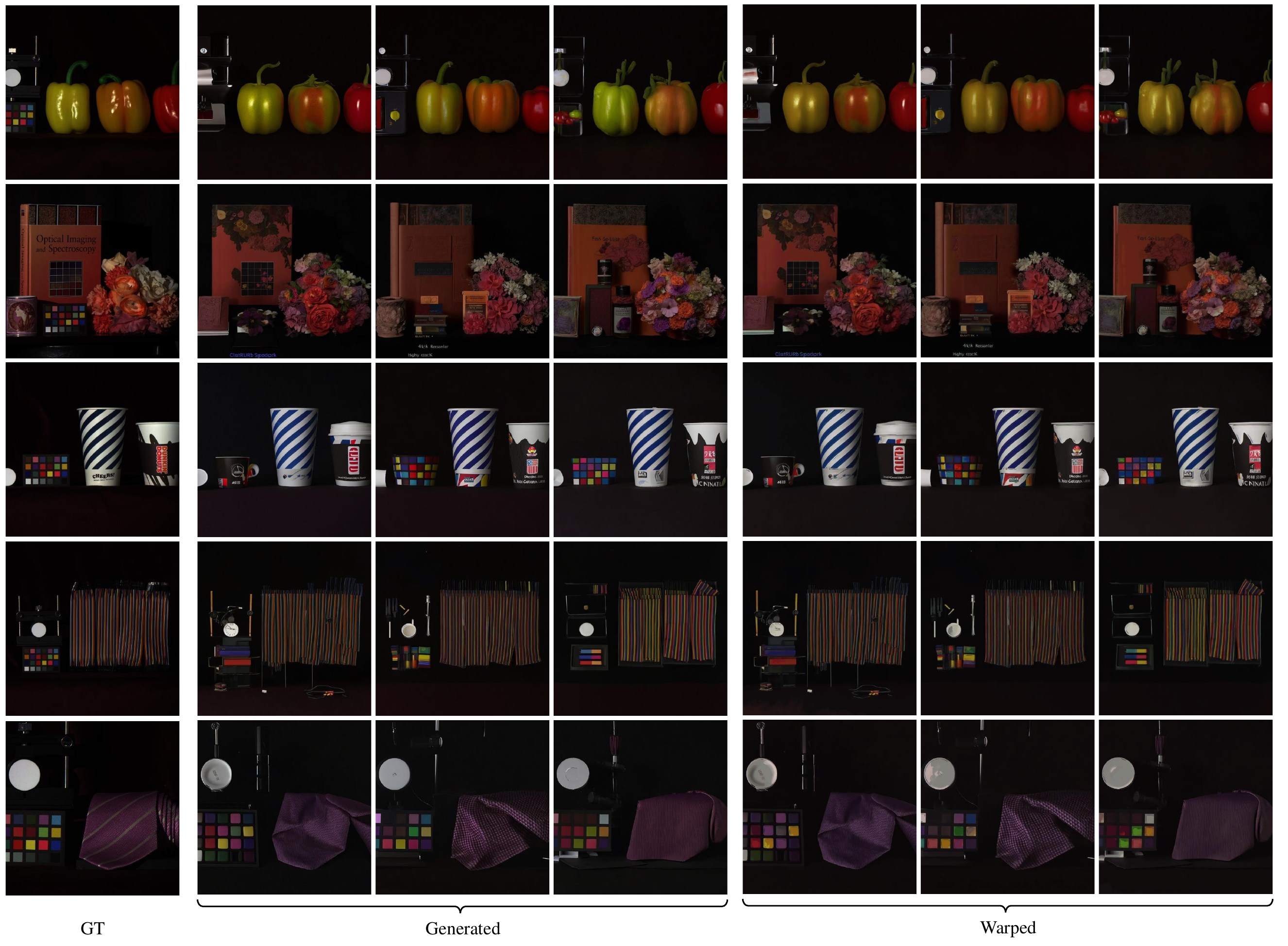}
    \caption{
    Additional generation and warped results on KAIST.
    }
    \label{fig:supp_kaist_generation_warped}
\end{figure*}
\begin{figure*}[!t]
    \centering
    \includegraphics[width=\textwidth]{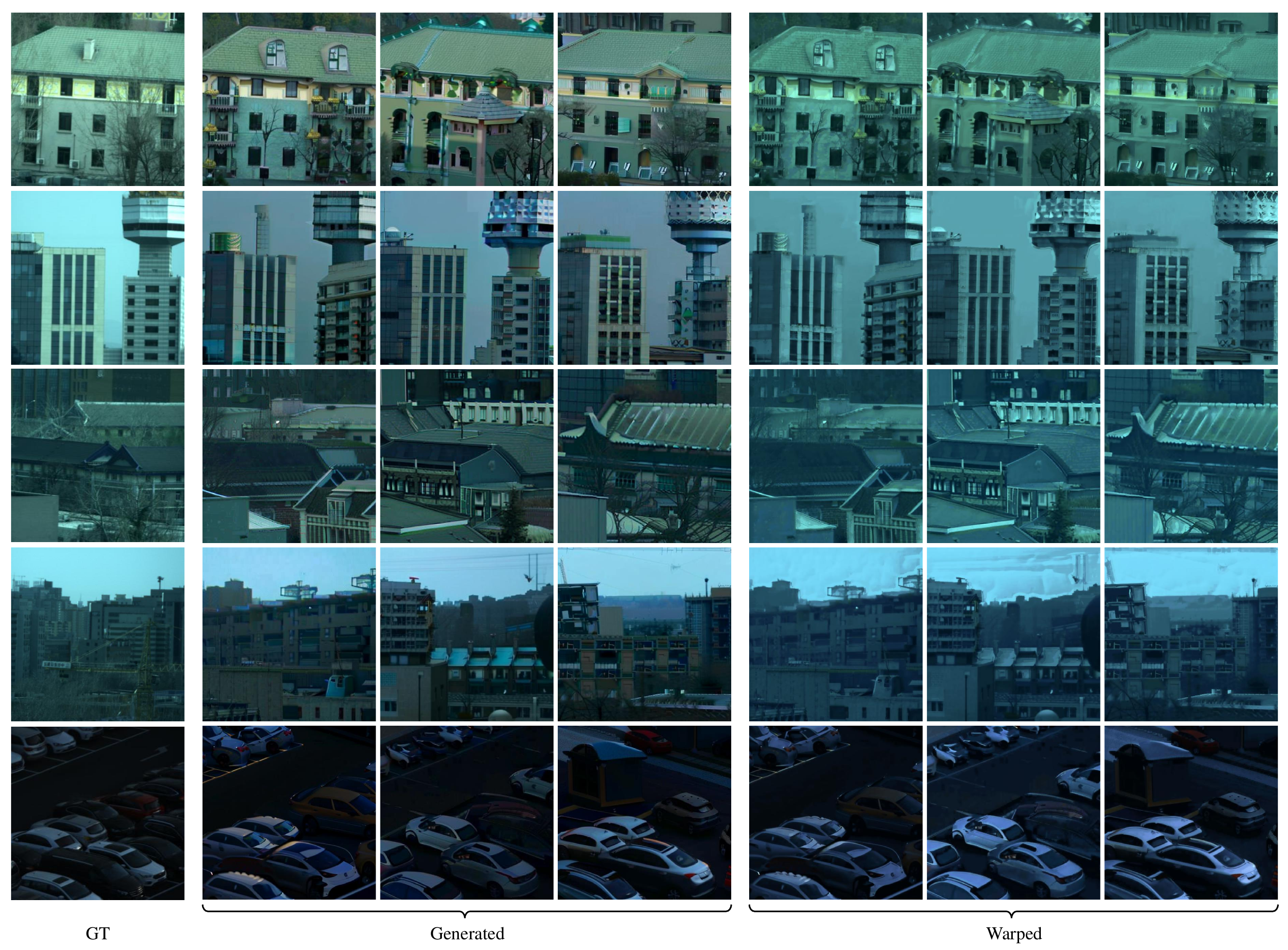}
    \caption{
    Additional generation and warped results on HSIDwrD.
    }
    \label{fig:supp_hsidwrd_generation_warped}
\end{figure*}
This section provides additional qualitative examples of the proposed distribution-adaptive synthesis module. In the main paper, we show representative generated samples and warped hyperspectral results to illustrate how HIR-ALIGN expands the target-domain training distribution. Here, we provide more examples on CAVE, KAIST, and HSIDwrD to further verify that the generation stage can produce diverse target-aligned appearances, while the warp-based spectral transfer module maps the proxy spectra to the generated layouts in a spatially coherent manner.

Fig.~\ref{fig:supp_cave_generation_warped} shows the generated and warped results on CAVE. The generated RGB views introduce variations in object layout, texture, illumination, and local appearance while still preserving scene-level semantics. After spectral warping, the corresponding hyperspectral samples remain visually aligned with the generated views, indicating that the estimated transport weights can transfer spectral information from the proxy HSI to the generated target-aligned layout.

Fig.~\ref{fig:supp_kaist_generation_warped} shows the corresponding visualization on KAIST. Compared with CAVE, KAIST contains more diverse indoor and outdoor scenes, including buildings, objects, and low-light environments. The results demonstrate that the proposed generation and warping pipeline is not restricted to a specific scene category. The generated samples remain semantically meaningful, and the warped results provide target-compatible hyperspectral supervision for subsequent finetuning.

Fig.~\ref{fig:supp_hsidwrd_generation_warped} further presents generation and warped results on HSIDwrD. Since HSIDwrD contains real-world hyperspectral observations, its images often include more complicated sensor characteristics and illumination variations than simulated benchmark data. The results show that HIR-ALIGN can still construct plausible generated views and corresponding warped hyperspectral samples, which supports its applicability to real-world target-domain adaptation.

\section{Additional Proxy Results}
\label{sec:supp_proxy_results}

This section visualizes the proxy results used by HIR-ALIGN across different restoration tasks. The proxy HSI is produced by applying a source-pretrained restorer to the degraded target observation, and it serves as the spectral anchor for both proxy-based finetuning and distribution-adaptive synthesis. Although the proxy is not assumed to be perfectly clean, it should preserve the dominant target-domain structure and approximate spectral information. Such a proxy provides a stable starting point for constructing target-aligned supervision.

Fig.~\ref{fig:supp_proxy_examples} presents proxy examples across six restoration tasks. For denoising tasks, SERT is used as the example proxy model. For super-resolution, MCNet is used. For the remaining tasks, including inpainting, deblurring, and band completion, MP-HSIR is used as the example proxy model. These examples show that the proxy branch can remove a substantial portion of the degradation while preserving scene content. This property is important because the downstream generated samples inherit their spectra from the proxy HSIs through the warp-based spectral transfer module.

\begin{figure*}[!t]
    \centering
    \includegraphics[width=\textwidth]{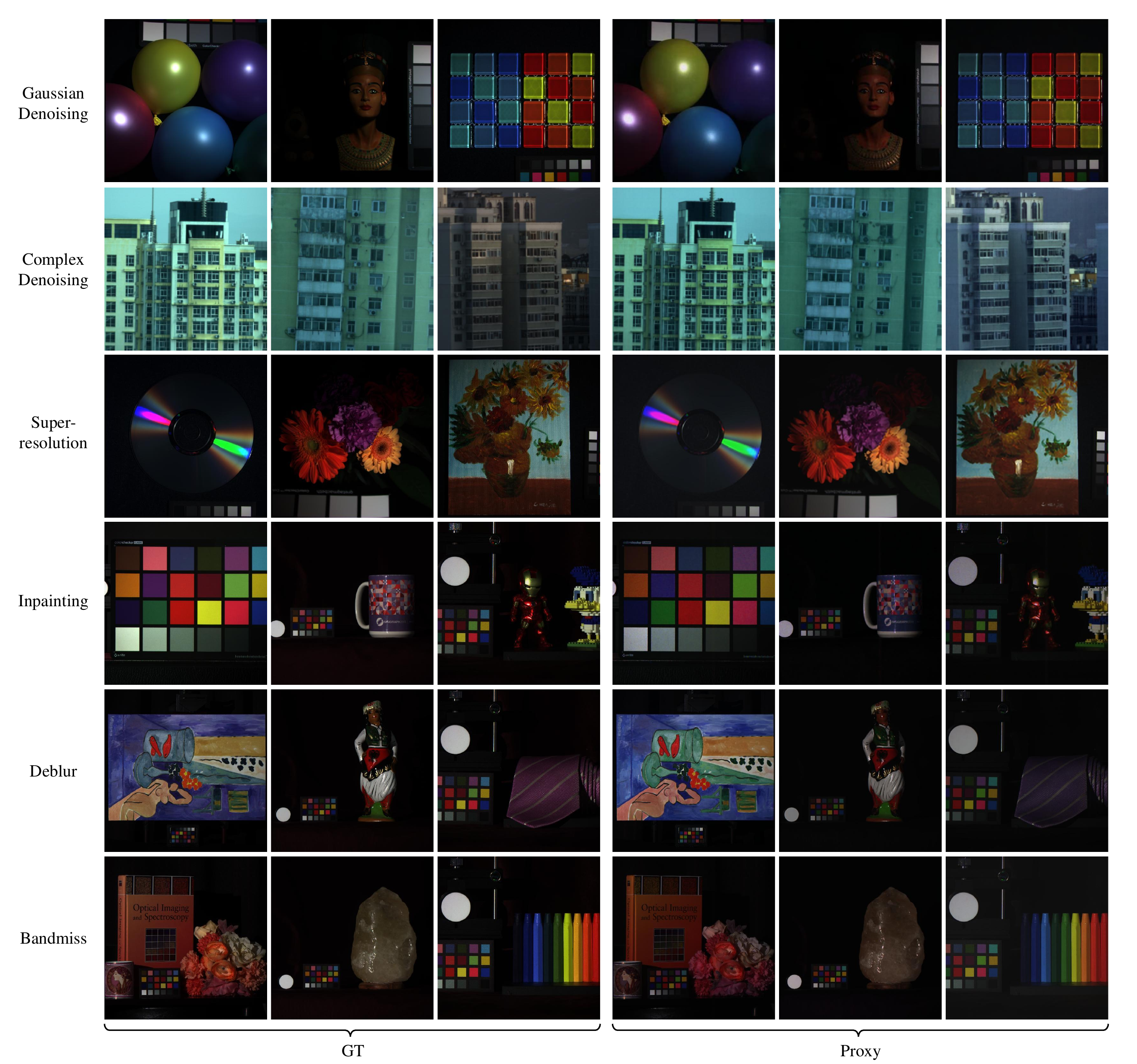}
    \caption{
    Proxy examples across different restoration tasks. The example models are SERT for denoising, MCNet for super-resolution, and MP-HSIR for the remaining tasks. The proxy results preserve the major scene structure and provide the spectral anchors used by HIR-ALIGN for target-adaptive data construction.
    }
    \label{fig:supp_proxy_examples}
\end{figure*}

\section{Additional Restoration Results}
\label{sec:supp_restoration_results}

This section provides additional qualitative restoration results of HIR-ALIGN. In the main paper, we report quantitative comparisons over multiple hyperspectral image restoration tasks. Here, we further visualize representative restored results to show the effectiveness of the proposed adaptation strategy across different datasets and degradation types.

Fig.~\ref{fig:supp_cave_restoration} shows the restoration results on CAVE. The examples cover six restoration tasks, demonstrating that HIR-ALIGN consistently improves the visual quality of restored hyperspectral images under different degradation settings. Compared with source-only restoration results, the adapted results better recover local structures, suppress degradation artifacts, and preserve more faithful spatial details. These observations indicate that the proposed target-adaptive finetuning can effectively improve restoration performance on the CAVE target domain.

\begin{figure*}[!t]
    \centering
    \includegraphics[width=0.8\textwidth]{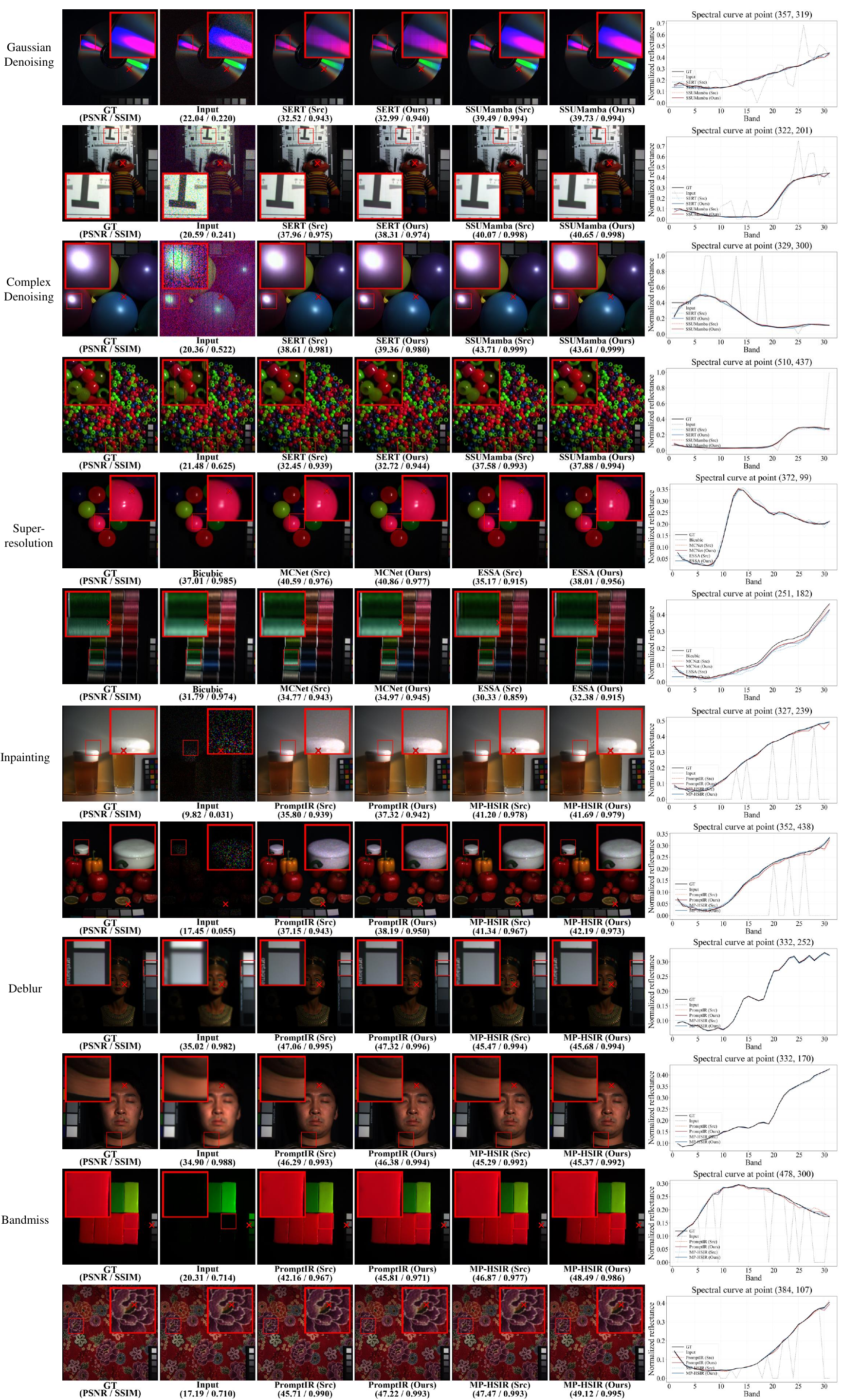}
    \caption{
    Additional restoration results on CAVE across six hyperspectral image restoration tasks. The adapted results produced by HIR-ALIGN recover clearer structures and reduce task-specific artifacts compared with source-only restoration results.
    }
    \label{fig:supp_cave_restoration}
\end{figure*}

Fig.~\ref{fig:supp_kaist_restoration} presents the corresponding restoration results on KAIST. Since KAIST contains more diverse scenes and imaging conditions, the visual comparisons further verify the generality of HIR-ALIGN. Across the six restoration tasks, the adapted models produce more visually consistent results and better preserve scene content. This suggests that the generated target-aligned supervision and proxy-anchored finetuning are effective beyond a single benchmark dataset.

\begin{figure*}[!t]
    \centering
    \includegraphics[width=0.8\textwidth]{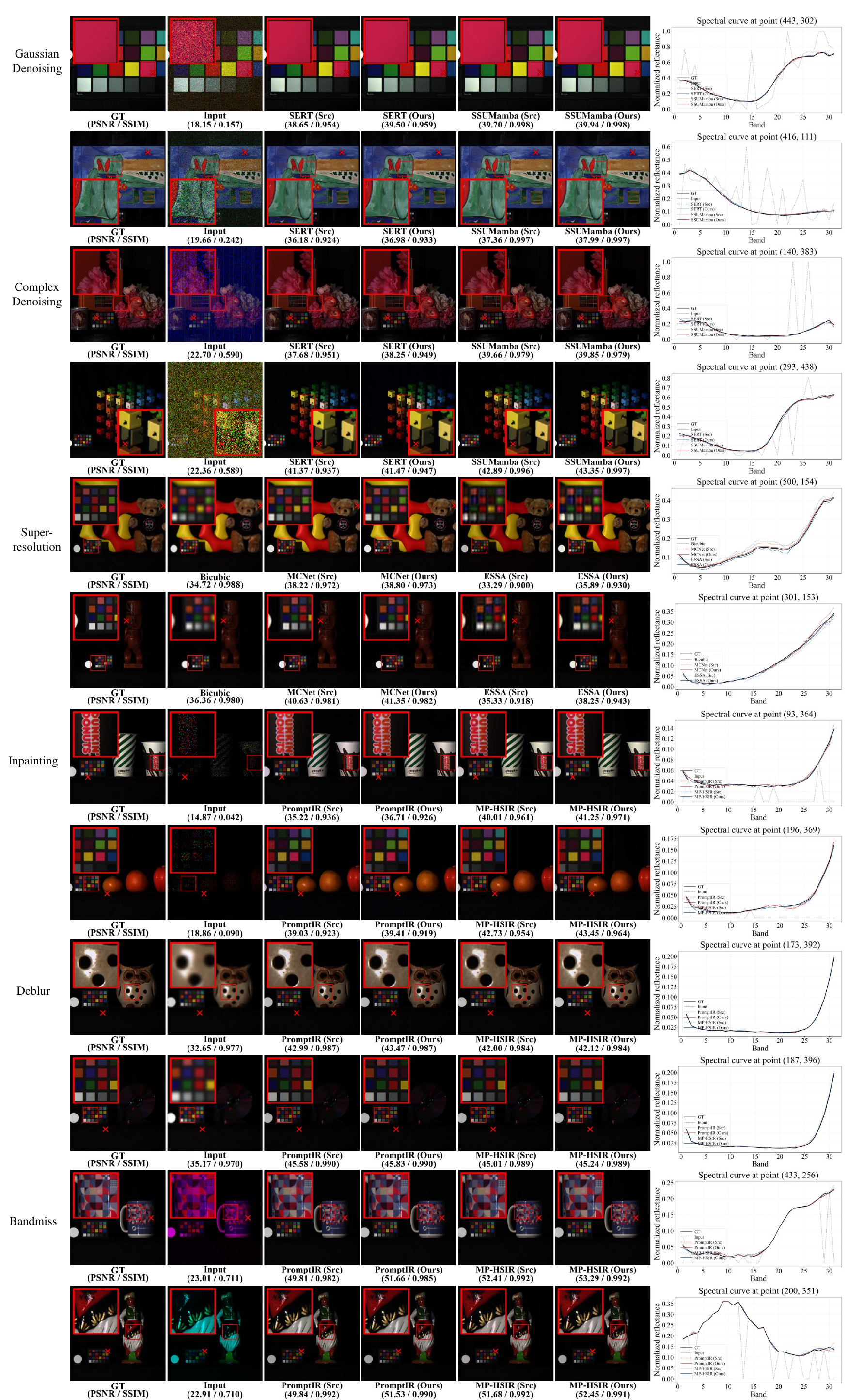}
    \caption{
    Additional restoration results on KAIST across six hyperspectral image restoration tasks. The results show that HIR-ALIGN improves restoration quality under diverse scene contents and degradation types.
    }
    \label{fig:supp_kaist_restoration}
\end{figure*}

Fig.~\ref{fig:supp_real_restoration} shows the denoising results on HSIDwrD. Compared with simulated benchmark data, real-world observations often contain more complex noise patterns and acquisition artifacts. The results demonstrate that HIR-ALIGN can still improve denoising quality in real scenarios, producing cleaner restored images while maintaining important spatial structures. This further supports the applicability of the proposed method to practical hyperspectral restoration problems.

\begin{figure*}[!t]
    \centering
    \includegraphics[width=0.7\textwidth]{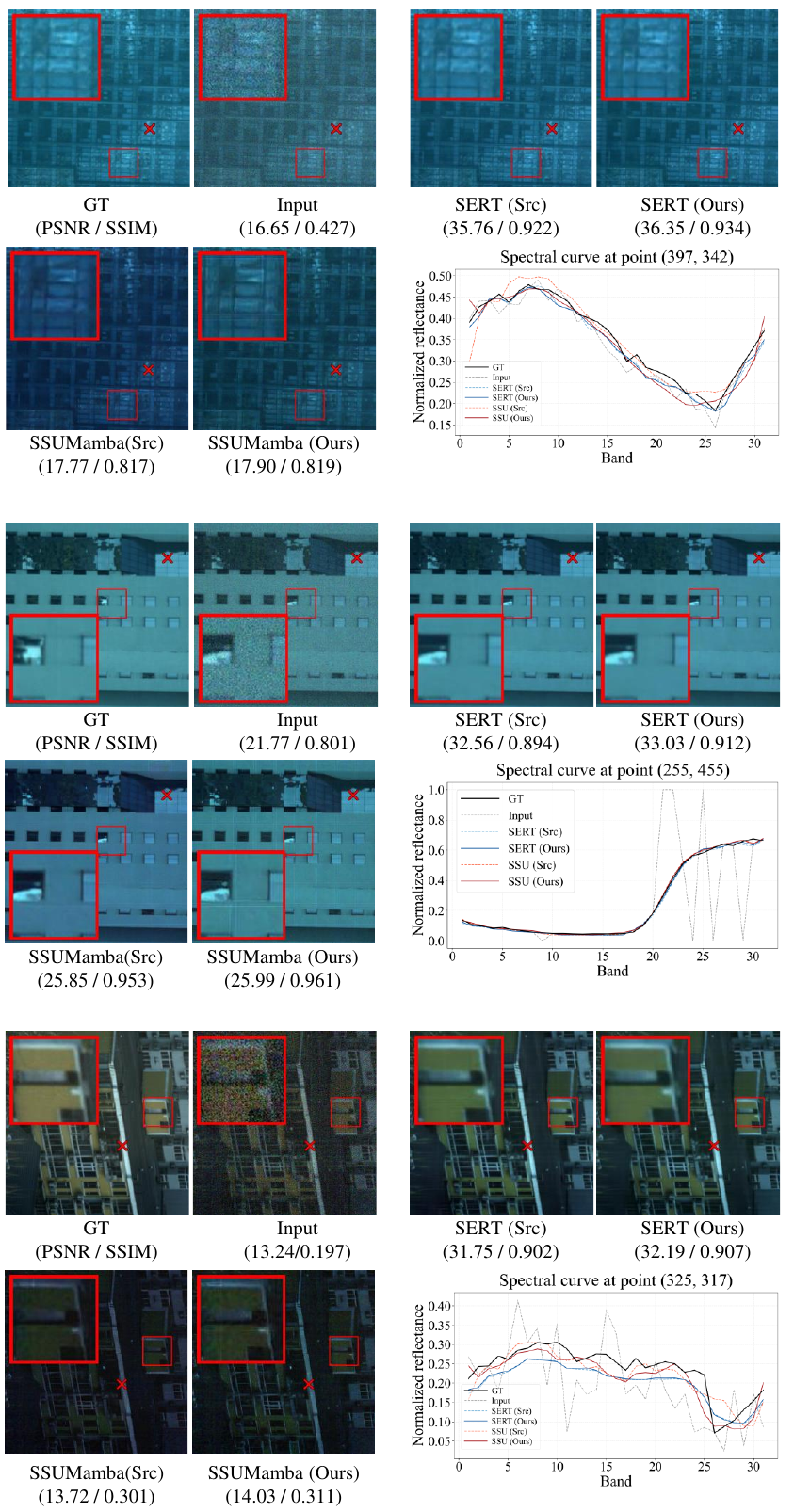}
    \caption{
    Additional denoising results on HSIDwrD. HIR-ALIGN suppresses real noise patterns while preserving spatial structures, showing its effectiveness under practical acquisition conditions.
    }
    \label{fig:supp_real_restoration}
\end{figure*}

\label{sec:supp_spectral_results}
\begin{figure*}[!t]
    \centering
    \includegraphics[width=\textwidth]{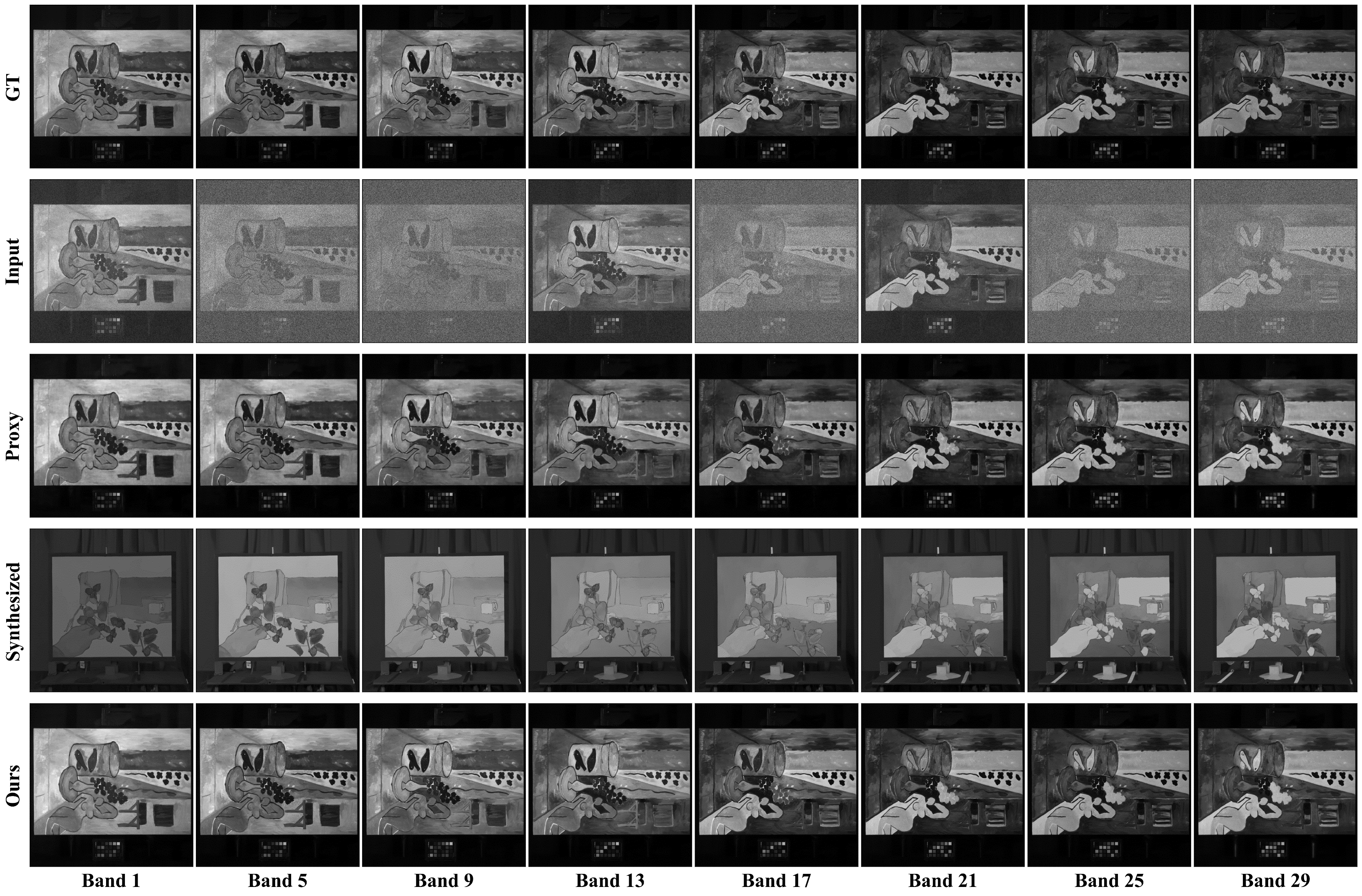}
    \caption{
    Band-wise visualization of representative hyperspectral images. The selected spectral bands show that the proposed method maintains more coherent structures across the spectral dimension and suppresses degradation artifacts in individual bands.
    }
    \label{fig:supp_band_examples}
\end{figure*}

\begin{figure*}[!t]
    \centering
    \includegraphics[width=\textwidth]{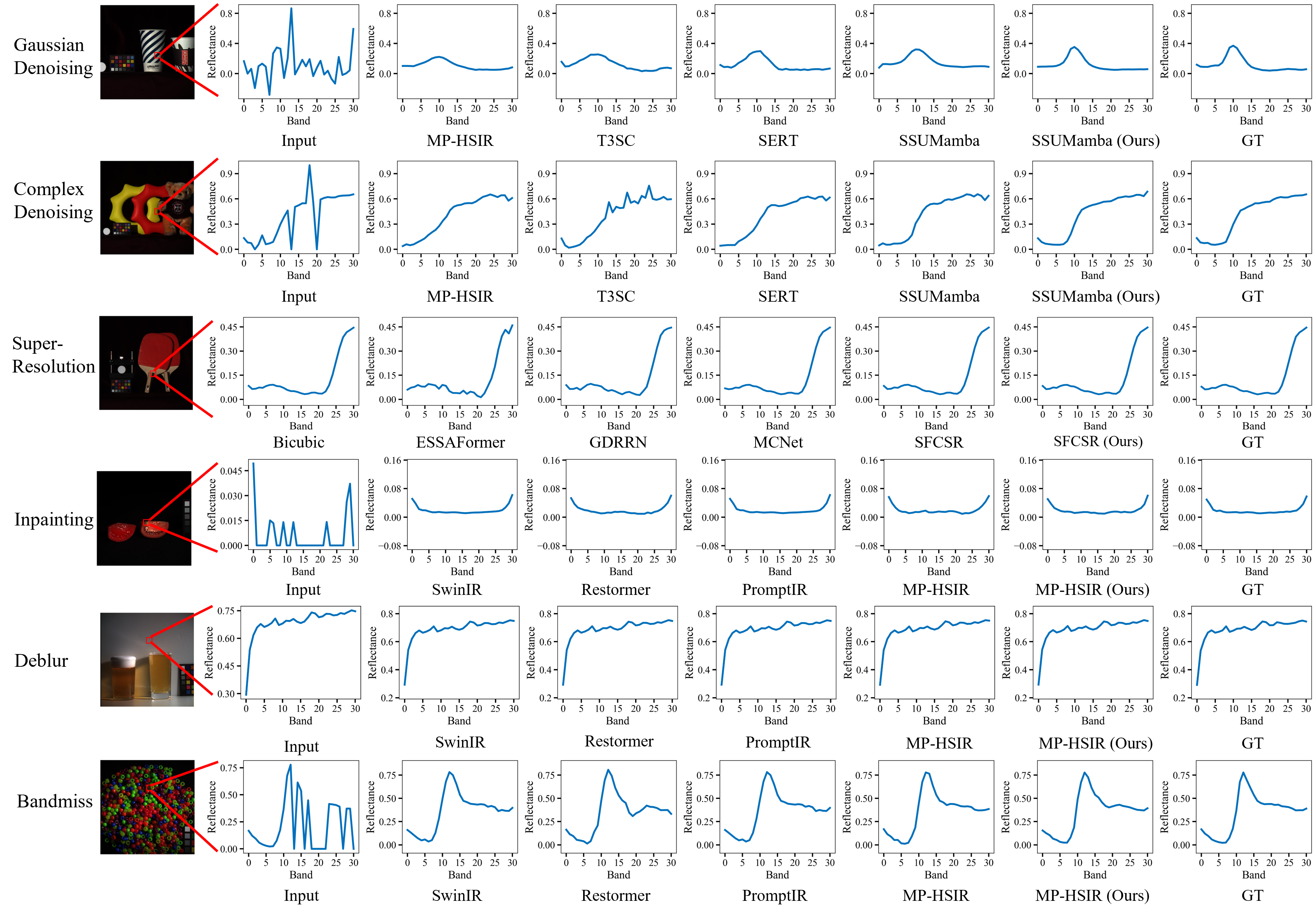}
    \caption{
    Spectral curve examples at representative spatial locations. The curves demonstrate that the adapted outputs better match the ground-truth spectral trends, showing improved spectral fidelity across bands.
    }
    \label{fig:supp_spectral_curve_examples}
\end{figure*}

Since hyperspectral restoration requires consistency over the full spectral cube, RGB-like visualizations alone are insufficient for evaluating restoration quality. We therefore provide additional spectral visualizations to show that the proposed adaptation improves not only spatial appearance but also spectral behavior. These results complement the quantitative PSNR, SSIM, and SAM comparisons in the main paper.

Fig.~\ref{fig:supp_band_examples} shows representative band-wise visualizations of a hyperspectral image. Each column corresponds to a selected spectral band, and each row presents a different source or restoration output. The comparison illustrates how degradation and restoration affect individual spectral bands. The proposed method produces more coherent spatial structures across different wavelengths, indicating that the restored hyperspectral cube is not merely visually plausible in RGB space but also consistent across the spectral dimension.

Fig.~\ref{fig:supp_spectral_curve_examples} further compares spectral curves at representative spatial locations. The marked pixels indicate the positions where the spectra are extracted. Compared with degraded or source-only results, the adapted outputs better follow the ground-truth spectral trend across bands. This provides a direct spectral-level verification that HIR-ALIGN improves hyperspectral restoration beyond RGB appearance, and that the proxy-anchored generation strategy helps maintain spectral fidelity.

Overall, these supplementary visualizations provide additional qualitative evidence for the effectiveness of HIR-ALIGN. The generation and warped results show that the proposed synthesis branch can enlarge the target-domain support, the proxy examples show that source-pretrained restorers can provide useful target-domain spectral anchors, and the spectral visualizations verify that the final outputs remain consistent across individual bands and full spectral curves.

\bibliographystyle{IEEEtran}
\bibliography{references_supp}

\end{document}